\documentclass{article}

\usepackage{microtype}
\usepackage{graphicx}
\usepackage{booktabs} 

\usepackage{hyperref}

\usepackage{url}
\usepackage{times}
\usepackage{epsfig}
\usepackage{amsmath}
\usepackage{amssymb}
\usepackage{amsthm}
\usepackage{multirow}
\usepackage{soul}
\usepackage{footnote}
\usepackage{tablefootnote}
\usepackage{caption}
\usepackage{subcaption}

\usepackage[accepted]{icml2021}

\newtheorem{definition}{Definition}

\newtheorem{theorem}{Theorem}

\newtheorem{lemma}{Lemma}

\newtheorem{remark}{Remark}

\newcommand{\Mat}{\boldsymbol}
\newcommand{\Set}{\mathcal}
\newcommand{\G}{}
\newcommand{\real}{\mathbb{R}}

\DeclareMathOperator{\diag}{diag}

\def\Capacity{{\Gamma}}

\newcommand{\CBest}[1]{\textcolor[rgb]{0.84,0.305,0.418}{\textbf{#1}}}
\newcommand{\CBetter}[1]{\textcolor[rgb]{0.2,0.63,0.254}{\textbf{#1}}}

\newcommand{\jshi}[1]{\textcolor{black}{#1}}
\newcommand{\wangph}[1]{\textcolor{black}{{#1}}}

\icmltitlerunning{SoGCN: Second-Order Graph Convolutional Networks}

\begin{document}

\twocolumn[
\icmltitle{SoGCN: Second-Order Graph Convolutional Networks}



\icmlsetsymbol{equal}{*}

\begin{icmlauthorlist}
\icmlauthor{Peihao Wang}{equal,sist}
\icmlauthor{Yuehao Wang}{equal,sist}
\icmlauthor{Hua Lin}{uisee}
\icmlauthor{Jianbo Shi}{uisee,upenn}
\end{icmlauthorlist}

\icmlaffiliation{sist}{SIST, ShanghaiTech University}
\icmlaffiliation{uisee}{AI Lab, Uisee Technology Ltd.}
\icmlaffiliation{upenn}{CIS, University of Pennsylvania}

\icmlcorrespondingauthor{Peihao Wang}{wangph@shanghaitech.edu.cn}

\icmlkeywords{Machine Learning, ICML}

\vskip 0.3in
]



\printAffiliationsAndNotice{\icmlEqualContribution} 

\begin{abstract}
Graph Convolutional Networks (GCN) with multi-hop aggregation is more expressive than one-hop GCN but suffers from higher model complexity. Finding the shortest aggregation range that achieves comparable expressiveness and minimizes this side effect remains an open question.
We answer this question by showing that multi-layer second-order graph convolution (SoGC) is sufficient to attain the ability of expressing polynomial spectral filters with arbitrary coefficients. Compared to models with one-hop aggregation, multi-hop propagation, and jump connections, SoGC possesses filter representational completeness while being lightweight, efficient, and easy to implement. Thereby, we suggest that SoGC is a simple design capable of forming the basic building block of GCNs, playing the same role as $3 \times 3$ kernels in CNNs. We build our Second-Order Graph Convolutional Networks (SoGCN) with SoGC and design a synthetic dataset to verify its filter fitting capability to validate these points. For real-world tasks, we present the state-of-the-art performance of SoGCN on the benchmark of node classification, graph classification, and graph regression datasets.

\end{abstract}

\section{Introduction} \label{sec:intro}

Graph Convolutional Networks (GCNs) has gained popularity in recent years.
Researchers have shown that non-localized multi-hop GCNs \citep{liao2019lanczosnet, luan2019break, abu2019mixhop} have better performance than localized one-hop GCNs \citep{defferrard2016convolutional, kipf2016semi, wu2019simplifying}.
However, in Convolutional Neural Networks (CNNs), the localized $3 \times 3$ kernels' expressiveness has been shown in image recognition both experimentally \citep{simonyan2014vggnet} and theoretically \citep{zhou2020universality}.
These contradictory observations motivate us to search for a maximally localized Graph Convolution (GC) kernel with guaranteed feature expressivness.

\begin{table}[tb!]
\centering
\small
\renewcommand{\arraystretch}{1.2}
\scalebox{0.95}{
\begin{tabular}{l c c c}
\hline
\textbf{Kernel} & \textbf{Expressiveness} & \textbf{Localized} & \textbf{Complexity} \\
\hline
Vanilla GCN &  Very Low & $\checkmark$ & $O(ms)$ \\
GIN  & Medium & $\checkmark$ & $O(ms)$ \\
Multi-hop  & Full & $\times$ & $O(Kms)$ \\ [0.5ex]
\hline
SoGCN (Ours)& Full & $\checkmark$ &  $O(ms)$ \\
\hline
\end{tabular}
}
\caption{Comparison of different GC kernels in terms of expressivness, localization, and time complexity. In this table, $m$ represents the number of neighborhoods around a graph node, $s$ is the dimension of input features, and $K$ denotes the aggregation length of multi-hop GCs. We compute the time complexity with respect to the method of \citet{abu2019mixhop}. }
\label{tbl:comparison}
\end{table}

Most existing GCN layers adopt localized graph convolution based on one-hop aggregation scheme as the basic building block \citep{kipf2016semi, hamilton2017inductive, xu2018powerful}.
The effectiveness of these one-hop models is based on the intuition that a richer class of convolutional functions can be recovered by stacking multiple one-hop layers.
However, extensive works \citep{li2018deeper, oono2019graph, cai2020note} have shown performance limitations of such design, which indicates this hypothesis may not hold.
\citet{liao2019lanczosnet, luan2019break, abu2019mixhop} observed that multi-hop aggregation run in each layer could lead to significant improvement in prediction accuracy.  However, a longer-range aggregator introduces extra hyperparameters and higher computational cost. This design also contradicts the compositionality principle of deep learning that neural networks benefit from deep connections and localized kernels \citep{lecun2015deep}.



\jshi{
Recent studies point out that one-hop GCNs suffer from filter incompleteness \citep{hoang2019revisiting}. By relating low-pass filtering on the graph spectrum with over-smoothing, one can show one-hop filtering could lead to performance limitations.
One natural solution of adding a more complex graph kernel, such as multi-hop connections, seems to work in practical settings \citep{abu2019mixhop}.  The question is: ``what is the simplest graph kernel with the full expressive power of the graph convolution?"  \textit{We show that with a second-order graph kernel of ``two-hop'' connection, we could approximate any complicated graph relationships.}  Intuitively, it means we should extract a contrast between graph nodes that almost-connected (via their neighbor) vs. directly-connected.
}

\jshi{
We construct the two-hop graph kernel with second-order polynomials in an adjacency matrix and call it the Second-Order GC (SoGC).  
We show this Second-Order GC (SoGC) is the ``sweet spot" balancing localization and fitting capability.
To justify our conclusion, we introduce a Layer Spanning Space (LSS) framework to quantify the filter representation power of multi-layer GCs. Our LSS works by mapping GC filters' composition with arbitrary coefficients to polynomial multiplication (Section \ref{sec:gcn_power}).
}

Under this LSS framework, we can show that \textit{SoGCs can approximate any linear GCNs in channel-wise filtering} (Theorem \ref{thm:2nd_order}).
Vanilla GCN and GIN (first-order polynomials in adjacency matrix) cannot represent all polynomial filters in general;
multi-hop GCs (higher-order polynomials in adjacency matrix) do not contribute more expressiveness (Section \ref{sec:other_gcs}). 
In this sense, SoGC is the most localized GC kernel with the full representation power.

To validate our theory, we build our Second-Order Graph Convolutional Networks (SoGCN) by layering up SoGC layers (Section \ref{sec:sogcn}). We reproduce our theoretical results on a synthetic datasets for filtering power testing (Section \ref{sec:synthetic_exp}).
On the public benchmark datasets \citep{dwivedi2020benchmarking}, SoGCN using simple graph topological features consistently boosts the performance of our baseline model (i.e., vanilla GCN) comparable to the state-of-the-art GNN models (with more complex attention/gating mechanisms).
We also verify that our SoGCN fits extensive real-world tasks, including network node classification, super-pixel graph classification, and molecule graph regression (Section \ref{sec:benchmarks_exp}).

To our best knowledge, this work is the first study that identifies the distinctive competence of the two-hop neighborhood in the context of expressing a polynomial filter with arbitrary coefficients.
Our SoGC is a special but non-trivial case of polynomial approximated graph filters \citep{defferrard2016convolutional}.
\citet{kipf2016semi} conducted an ablation study with GC kernels of different orders but missed the effectiveness of the second-order relationships.  The work of \citet{abu2019mixhop} talked about muti-hop graph kernels; however, they did not identify the critical importance of the two-hop form.  In contrast, we clarify the prominence of SoGCs in theories and experiments.

\jshi{
Our research on GCN using pure topologically relationship is orthogonal to those using geometric relations \citep{monti2017geometric, fey2018splinecnn, pei2020geom}, or those with expressive edge features \citep{li2016gated, gilmer2017neural}, and hyper-edges \citep{morris2019weisfeiler, maron2018invariant, maron2019provably}.  It is also independent with graph sampling procedures \citep{rong2019dropedge, hamilton2017inductive}.
}

\begin{figure*}[tb] 
	\begin{subfigure}[b]{0.33\textwidth}
		\centering
		\includegraphics[width=0.9\textwidth]{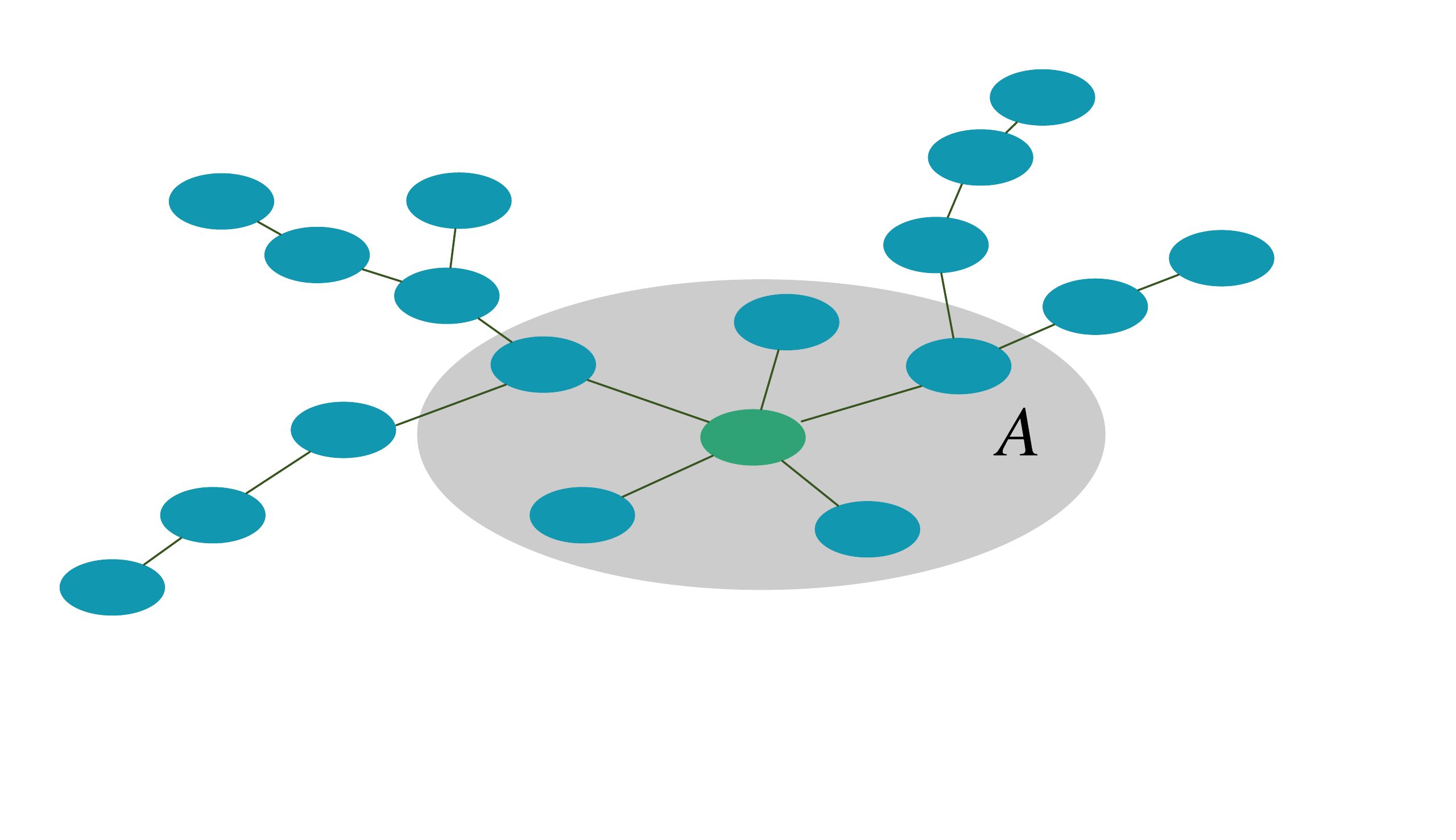}
		\caption{Vanilla GC}
		\label{fig:vgc}
	\end{subfigure}
	\begin{subfigure}[b]{0.33\textwidth}
		\centering
		\includegraphics[width=0.9\textwidth]{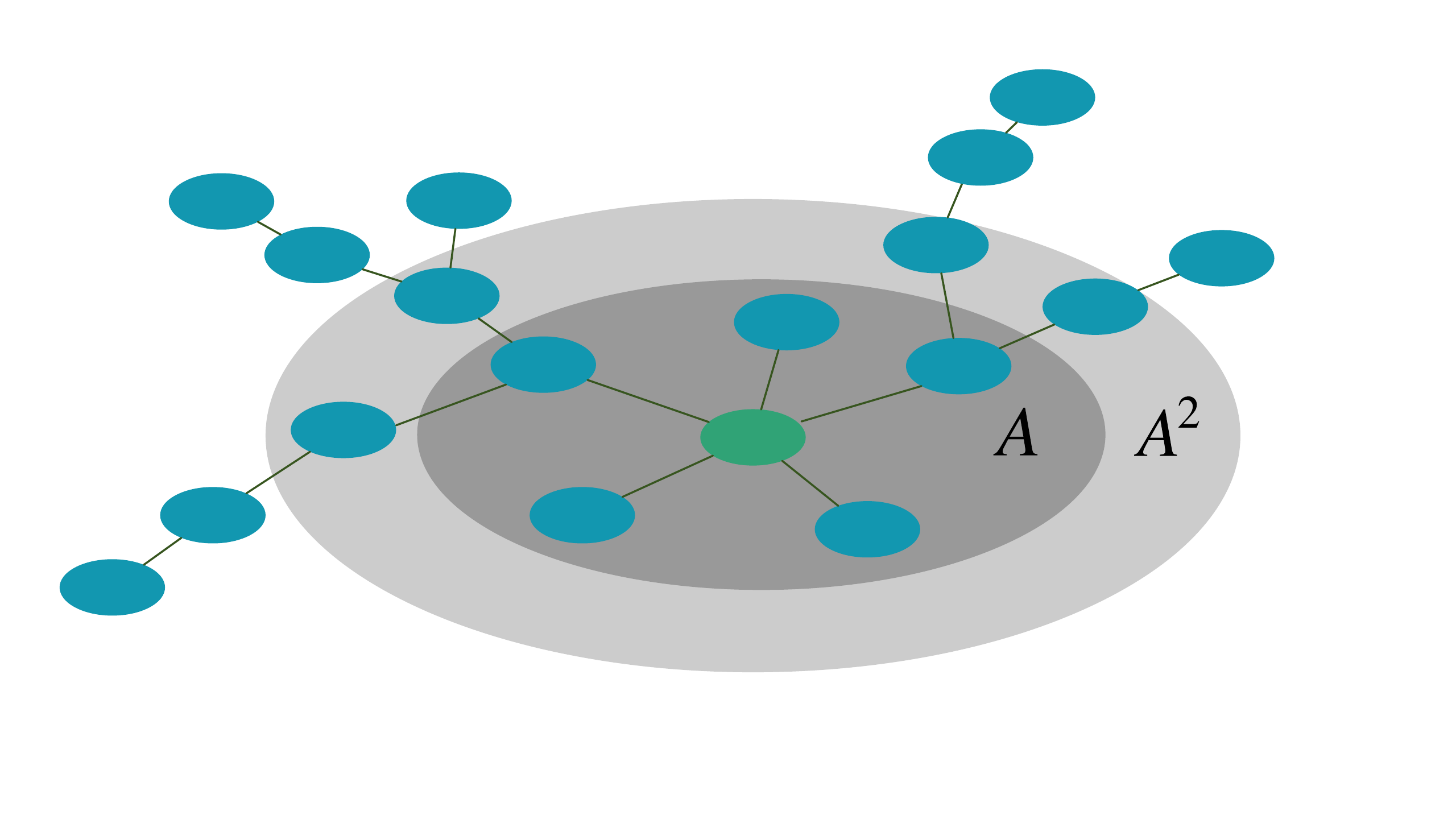}
		\caption{Our SoGC}
		\label{fig:sogc}
	\end{subfigure}
	\begin{subfigure}[b]{0.33\textwidth}
		\centering
		\includegraphics[width=0.9\textwidth]{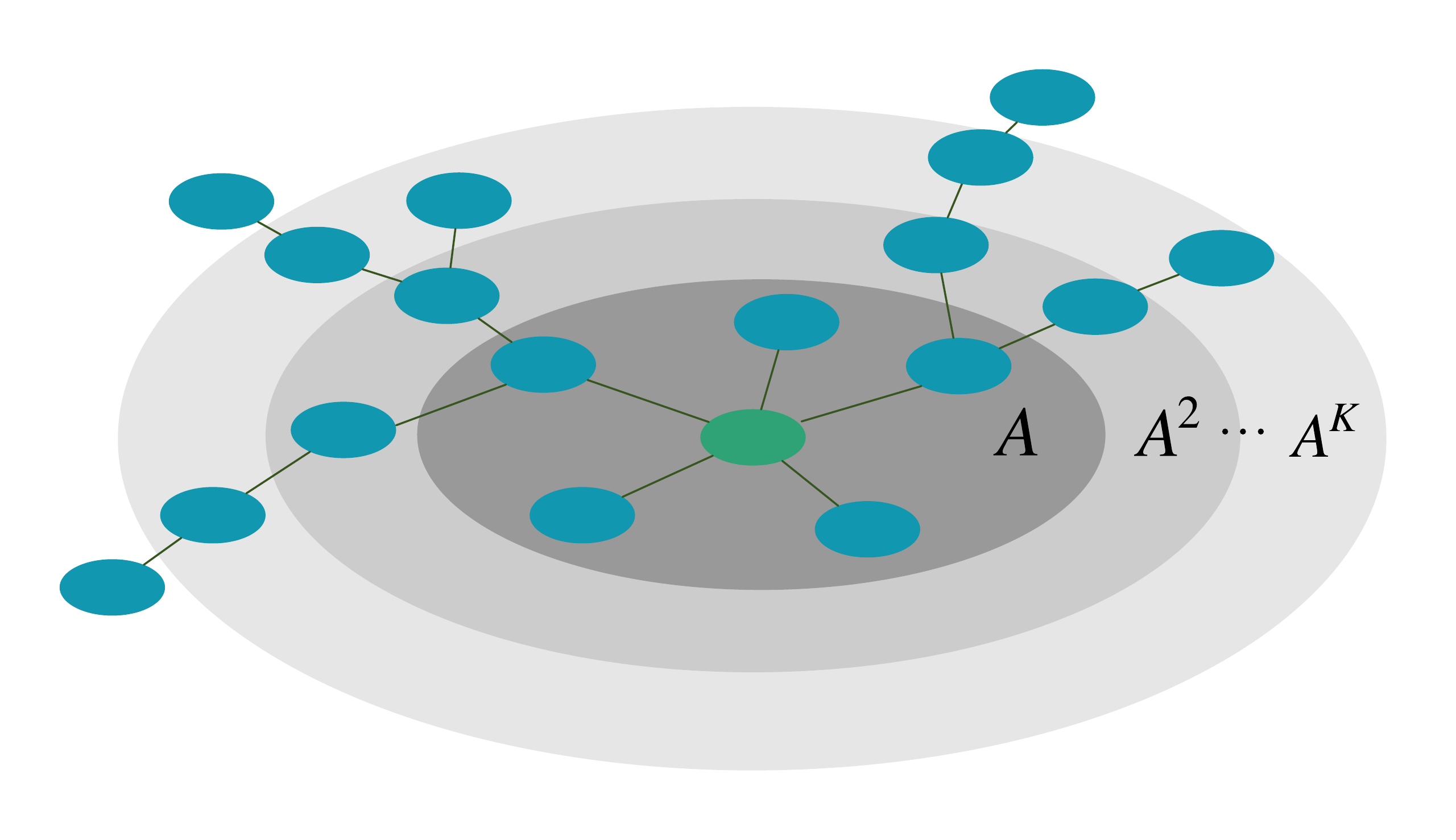}
		\caption{Multi-Hop GC}
		\label{fig:khopgc}
	\end{subfigure}
	
	\caption{
		Vertex domain interpretations of vanilla GC, SoGC, and Multi-Hop GC.
	Denote $A$ the first-hop aggregator, $A^2$ the second-hop aggregator, and $A^K$ the $K$-th hop aggregator. Nodes in the same colored ring share the same weights.
		(a) Vanilla GC only aggregates information from the first-hop neighbor nodes.
		(b) SoGC incorporates additional information from the second-hop (almost-connected) neighborhood.
		(c) Multi-hop GC simply repeats mixing information from every neighborhood within $K$ hops.}
	\label{fig:gcs}
\end{figure*}

\section{Related Work}

\paragraph{Spectral GCNs.}
Graph convolution is defined as element-wise multiplication on graph spectrum \citep{hammond2011wavelets}.
\citet{bruna2013spectral} first proposed spectral GCN with respect to this definition.
ChebyNet \citep{defferrard2016convolutional} approximates graph filters using Chebyshev polynomials.
Vanilla GCN \citep{kipf2016semi, wu2019simplifying} further reduces the GC layer to a degree-one polynomial with lumping of first-order and constant terms.
GIN \citep{xu2018powerful} disentangles the effect of self-connection and pairwise neighboring connections by adding a separate mapping for central nodes.
However, these simplifications causes performance limitations \citep{oono2019graph, cai2020note}.
APPNP \citep{klicpera2018predict} uses Personalized PageRank to derive a fixed polynomial filter.
\citet{bianchi2019graph} proposes a multi-branch GCN architecture to simulate ARMA graph filters.
GCNII \citep{chenWHDL2020gcnii} incorporates identity mapping and initial mapping to relieve over-smoothing problem and deepen GCNs.
However, these models are not easy to implement and introduce additional hyper-parameters.


\paragraph{Multi-Hop GCNs.}
To exploit multi-hop information, \citet{liao2019lanczosnet} proposes to use Lanczos algorithm to construct low rank approximations of the graph Laplacian for graph convolution.
\citet{luan2019break} devises two architectures Snowball GCN and Truncated Krylov GCN to capture neighborhoods at various distances.
To simulate neighborhood delta functions, \citet{abu2019mixhop} repeat mixing multi-hop features to identify more topological information.
JKNet \citep{xu2018jknet} combines all feature activation of previous layers to learn adaptive and structure-aware representations of different graph substructures.
These models exhibit the strength of multi-hop GCNs over one-hop GCNs while leaving the propagation length as a hyper-parameter.
In the meanwhile, long-range aggregation in each layer causes higher complexity (Table \ref{tbl:comparison}).

\paragraph{Expressiveness of GCNs.}
\wangph{
Most of the works on GCN's expressiveness are restricted to the over-smoothing problem: \citet{li2018deeper} first poses the over-smoothing problem; \citet{hoang2019revisiting} indicates GCNs are no more than low-pass filters; \citet{luan2019break,oono2019graph} demonstrate the asymptotic behavior of feature activation to a subspace; \citet{cai2020note} examines the decreasing Dirichlet energy.
These analytic frameworks do not provide constructive improvements on building more powerful GCN.
\citet{chenWHDL2020gcnii} first proposed to assess GCN's overall expressiveness by tackling the ability of expressing polynomial filters with arbitrary coefficients. But their theory is only applicable for transductive learning on a single graph, and does not upper bound the degree of the graph polynomial filters.
}

\section{Second-Order Graph Convolution} \label{sec:sogc}

We begin by introducing our notation.
We are interested in learning on a finite graph set $\Set{G} = \left\{ \G{G}_1, \cdots, \G{G}_{\lvert \Set{G} \rvert} \right\}$.
Assume each graph $\G{G} \in \Set{G}$ is simple and undirected, associated with a finite vertex set $\Set{V}(\G{G})$, an edge set $\Set{E}(\G{G}) \subseteq \Set{V}(\G{G}) \times \Set{V}(\G{G})     $, and a symmetric normalized adjacency matrix $\Mat{A}(\G{G})$ \citep{chung1997spectral, shi2000normalized}.
Without loss of generality and for simplicity, $\lvert \Set{V}(\G{G}) \rvert = N$ for every $\G{G} \in \Set{G}$.
We denote single-channel signals supported in graph $\G{G} \in \Set{G}$ as $\Mat{x} \in \real^N$, a vectorization of function $\Set{V}(\G{G}) \rightarrow \real$.

Graph Convolution (GC) is defined as Linear Shift-Invariant (LSI) operators to adjacency matrices \citep{sandryhaila2013discrete}.
This property enables GC to extract features regardless of where the local structure falls.
A single-channel GC can be written as a mapping $f : \Set{G} \times \real^N \rightarrow \real^N$.  According to \citet{defferrard2016convolutional}, a GC can be approximated by a polynomial in adjacency matrix (Figure \ref{fig:khopgc}) \footnote{We can replace the Laplacian matrix $\Mat{L}$ in \citet{defferrard2016convolutional} with the normalized adjacency matrix $\Mat{A}$ since $\Mat{L} = \Mat{I} - \Mat{A}$.}:
\begin{equation} \label{equ:spectral_gc}
	f_{\Mat{\theta}}(\G{G}, \Mat{x})  = \sum_{k=0}^{K} \theta_{k} \Mat{A}(\G{G})^k \Mat{x},
\end{equation}
where $\Mat{\theta} = \begin{bmatrix} \theta_0 & \cdots & \theta_K \end{bmatrix}^T \in \real^{K+1}$ represents the kernel weights.
The Equation \ref{equ:spectral_gc} indicates that graph convolution can be interpreted as a linear combination of features aggregated by $\Mat{A}(G)^k$.  Thereby, the hyperparameter $K$ can reflect the localization of a GC kernel.

Previous work of \citet{kipf2016semi} simplified the Equation \ref{equ:spectral_gc} to a one-hop kernel, so-called vanilla GC (Figure \ref{fig:vgc}).
We formulate vanilla GC as below:
\begin{align}
	\label{equ:vgc} & f_{1}(G, \Mat{x}) = \theta \left( \Mat{A}(G) + \Mat{I} \right) \Mat{x}.
\end{align}
We are interested in the overall graph convolution networks' representation power of expressing a polynomial filter (cf. Equation \ref{equ:spectral_gc}) with arbitrary degrees and coefficients. At first glance, one-hop GC (cf. Equation \ref{equ:vgc}) can approximate any high-order GC kernels by stacking multiple layers. However, that is not the case (See formal arguments in Section \ref{sec:other_gcs}).
In contrast, when plugging the second-order term into vanilla GC, we will show that this approximation ability can be attained (See formal arguments in Theorem \ref{thm:2nd_order}). We name this improved design Second-Order Graph Convolution (SoGC), as it can be written as the second-order polynomial in adjacency matrix:
\begin{equation} \label{equ:sogc}
	f_{2}(G, \Mat{x}) = \left( \theta_2 \Mat{A}(G)^2 + \theta_1 \Mat{A}(G) + \theta_0 \Mat{I} \right) \Mat{x}.
\end{equation}
We illustrate its vertex-domain interpretation in Figure \ref{fig:sogc}.
The critical insight is that graph filter approximation can be viewed as a polynomial factorization problem. It is known that \textit{any univariate polynomial can be factorized into sub-polynomials of degree two}.  Based on this fact, we show by stacking enough SoGCs (and varying their parameters) can achieve decomposition of any polynomial filters.

\begin{figure*}[tb] 
	\centering
	\includegraphics[width=\textwidth]{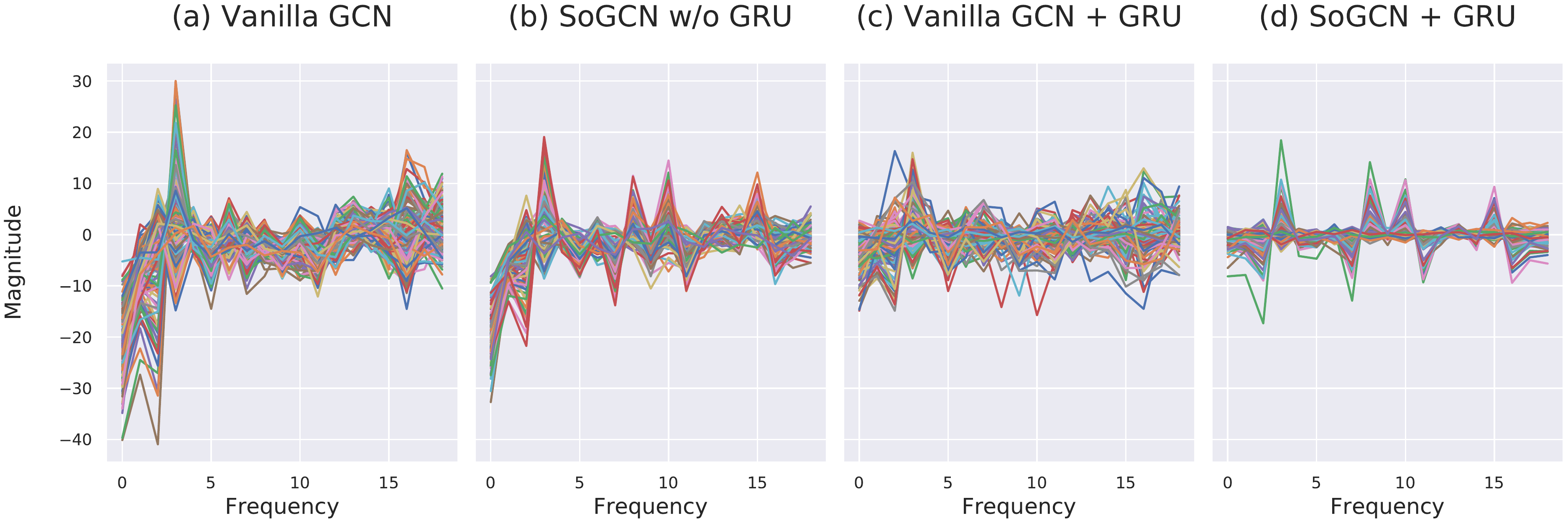}
	\caption{
	Visualizing output activation in graph spectrum domain for vanilla GCN, SoGCN, and GRU variants. The test is conducted on a graph from the ZINC dataset. The spectrum is defined as a projection of activation functions on the graph eigenvectors. SoGCN preserved higher-order spectrum, while vanilla GCN shows over-smoothing. See Appendix \ref{supp:more_vis_spectrum} for more visualizations on the ZINC dataset.
	}
	\label{fig:vis_spectrum}
\end{figure*}

\subsection{Representation Power Justification} \label{sec:gcn_power}

To be more precise, we introduce our Layer Spanning Space (LSS) framework in this subsection and mathematically prove that arbitrary GC kernel can be decomposed into finite many SoGC kernels.

First, to illustrate overall graph filter space and filters expressed by a single layer, we define a \textit{graph filter space}, in which every polynomial filter has degree no more than $K$:
\begin{definition}{(Graph filter space)} \label{dfn:filter_space}
	Suppose the parameter space is $\real$. The Linear Shift-Invariant (LSI) graph filter space\footnote{See Appendix \ref{supp:remark_def1} for a remark on how $\Set{F}_K$ relates with the term ``Linear Shift-Invariant".} of degree $K > 0$ with respect to a finite graph set $\Set{G}$ is defined as $\Set{F}_K = \left\{ f_{\Mat{\theta}} : \Set{G} \times \real^N \rightarrow \real^N, \forall \Mat{\theta} \in \real^{K+1} \right\}$,
	where $f_{\Mat{\theta}}$ follows the definition in Equation \ref{equ:spectral_gc}.
\end{definition}
We further provide Definition 2 and Lemma \ref{lem:dim} to discuss the upper limit of $\Set{F}_K$'s dimension.
\begin{definition}{(Spectrum capacity)} \label{dfn:spectrm_capacity}
	Let spectrum set $\Set{S}({\Set{G}}) = \{ \lambda : \lambda \in \Set{S}(\Mat{A}(\G{G})), \forall \G{G} \in \Set{G} \}$, where $\Set{S}(\Mat{A})$ denotes the eigenvalues of $\Mat{A}$.
	Spectrum capacity $\Capacity = \lvert \Set{S}(\Set{G}) \rvert$ is the cardinality of all distinct graph eigenvalues.
	In particular, $\Capacity = (N-1) \lvert \Set{G} \rvert$ if every graph adjacency matrix has no common eigenvalues other than $1$.
\end{definition}

\begin{lemma} \label{lem:dim}
	Filter space $\Set{F}_K$ with degree $K > 0$ has dimension $\min \{ K+1, \Capacity \}$ as a vector space.
\end{lemma}
Lemma \ref{lem:dim} follows from Theorem 3 of \citet{sandryhaila2013discrete}.
See the complete proof in Appendix \ref{supp:prf_dim}. $\qed$

According to Lemma \ref{lem:dim}, one can define an ambient filter space $\Set{A}$ of degree $\Capacity-1$ (i.e., $\Set{F}_{\Capacity-1}$).  Suppose a GCN consists of $K$-hop convolutional kernels $f^{(l)}, l = 1, \cdots, L$, where the superscript indicates the layer number, and $L$ denotes the network depth. We can consider each layer is sampled from $\Set{F}_K$. We intend to justify this GCN's filter representation power via its Layer Spanning Space (LSS):
\begin{equation} \label{equ:khop_lss}
    \Set{F}_K^L = \left\{ F : F = f^{(1)} \circ \cdots \circ f^{(L)}, \forall f^{(l)} \in \Set{F}_K \right\},
\end{equation}
where the whole LSS is constructed by varying parameters of $f^{(l)}$ over $\real^{K+1}$. When $\Set{A} \subseteq \Set{F}_K^L$, the LSS covers the entire ambient space, then we say the GCN composed of $L$ GC kernels in $\Set{F}_K$ has \textit{full filter representation power}.

As $\Set{F}_K$ is a function space, $\Set{F}_K^L$ can be analytically tricky. To investigate this space, we define a mapping $\tau: \Set{F}_K \rightarrow \real_K[x]$, where $\real_K[x]$ denotes a polynomial vector space of degree at most $K$:
\begin{equation} \label{equ:ring_iso}
    \tau: \sum_{k=0}^{K} \theta_{k} \Mat{A}(\G{G})^k \mapsto \sum_{k=0}^{K} \theta_{k} x^k,
\end{equation}
We provide Lemma \ref{lem:tau} to reveal a good property of $\tau$:
\begin{lemma} \label{lem:tau}
	$\tau$ is a ring isomorphism when $K \le \Capacity - 1$.
\end{lemma}
The proof can be found in Appendix \ref{supp:prf_ring_iso}. $\qed$

This ring isomorphism signifies that the composition of filters in $\Set{F}_K$ is identical to polynomial multiplication. Therefore, one can study the LSS through the polynomials that can be factorized into the corresponding sub-polynomials. For example, Equation \ref{equ:khop_lss} is identical to:
\begin{equation}
	\Set{F}_K^L \simeq \left\{p(x) : p(x) = \prod_{l=1}^{L} \sum_{k=0}^{K} \theta^{(l)}_k x^k, \theta^{(l)}_k \in \real \right\}.
\end{equation}

In the rest of this subsection, we will show that $\lceil \frac{\Capacity - 1}{2} \rceil$-layer SoGCs (i.e., $\Set{F}_2$) can attain full representation power. That is, $\Set{F}_2^L$ covers the whole ambient filter space $\Set{A}$ when $L$ is as large as $\lceil (\Capacity - 1) / 2 \rceil$.
We summarize a formal argument in the following theorem:
\begin{theorem} \label{thm:2nd_order}
	For any $f \in \Set{A}$, there exists $f_2^{(l)} \in \Set{F}_2$ with coefficients $\theta_0^{(l)}, \theta_1^{(l)}, \theta_2^{(l)} \in \real, l = 1, \cdots, L$ such that $f = f_2^{(L)} \circ \cdots \circ f_2^{(1)}$ where $L \le \lceil (\Capacity-1)/2 \rceil$.
\end{theorem}
\begin{proof}
Proving Theorem \ref{thm:2nd_order} requires a fundamental polynomial factorization theorem, rephrased as below:
\begin{lemma}{(Fundamental theorem of algebra)} \label{lem:irreducible_over_real}
	Over the field of reals, the degree of an irreducible non-trivial univariate polynomial is either one or two.
\end{lemma}

For any $f \in \Set{A}$, apply $\tau: \Set{A} \rightarrow \real_{\Capacity-1}[x]$ (cf. Equation \ref{equ:ring_iso}) to map kernel $f$ to polynomial $h(x)$. By Lemma \ref{lem:dim}, $\deg h(x) \le \Capacity-1$.
By Lemma \ref{lem:irreducible_over_real}, factorize $h(x)$ into series of polynomials with the degree at most two, and then merge first-order polynomials into second-order ones until one single or no first-order sub-polynomial remains. As a consequence, $h(x)$ can be written as $h(x) = \prod_{l=1}^{L} h_l(x)$, where $L = \lceil \deg h(x) / 2 \rceil$. If $\deg h(x)$ is even, $\deg h_l(x) = 2$ for every $l = 1, \cdots, D$. Otherwise, except for at most one $h_l(x)$ whose degree is one, all terms have degree two.

The last step is to apply the inverse of morphism $\tau^{-1}: \real_{\Capacity-1}[x] \rightarrow \Set{A}$ formulated as below:
\begin{equation*}
	\tau^{-1}: \sum_{k = 0}^{K} \theta_k x^k \mapsto \sum_{k = 0}^{K} \theta_k \Mat{A}(\G{G})^k.
\end{equation*}
Since $\tau^{-1}$ is also a ring isomorphism, we have:
\begin{align*}
    \tau^{-1}(h(x)) &= \tau^{-1}(h_1(x) \cdots h_L(x)) \\
    &= \tau^{-1}(h_1(x)) \circ \cdots \circ \tau^{-1}(h_L(x)) \\
    &= f^{(1)} \circ \cdots \circ f^{(L)}
\end{align*}
where $f^{(l)} \in \Set{F}_{2}, l = 1, \cdots, L$ by definition, which implies $f = \tau^{-1}(\tau(f)) = f^{(L)} \circ \cdots \circ f^{(1)}$.
\end{proof}

Theorem \ref{thm:2nd_order} can be regarded as the universal approximation theorem of linear GCNs. Although nonlinear activation is not considered within our theoretical framework, we make reasonable hypothesis that achieving linear filter completeness can also boost GCNs with nonlinearity \citep{chenWHDL2020gcnii}.
See our experiments in Section \ref{sec:benchmarks_exp}.

Theorem \ref{thm:2nd_order} implies that multi-layer SoGC can implement arbitrary filtering effects and extract features at any positions on the spectrum (Figure \ref{fig:vis_spectrum}). Theorem \ref{thm:2nd_order} also coincides with \citet{dehmamy2019understanding} on how GCNs built on SoGC kernels could utilize depth to raise fitting accuracy (Figure \ref{fig:deeper_on_bandpass}).

\subsection{Compared with Other Graph Convolution} \label{sec:other_gcs}

In this subsection, we will show that vanilla GCN (and GIN) does not attain full expressiveness in terms of filter representation. We will also contrast our SoGC to multi-hop GCs (i.e., higher-order GCs in our terminology) to further reveal the prominence of SoGC. A brief comparison is summarized in Table \ref{tbl:comparison}.

\paragraph{Vanilla vs. second-order.} \label{sec:vgc_vs_sogc}
Vanilla GCN \citep{kipf2016semi} is a typical one-hop GCN with lumbing of of graph node self-connection and pairwise neighboring connection.
Compared with SoGC, vanilla GC is more localized and computationally cheaper.
However, this design has huge performance limitations \citep{hoang2019revisiting, wu2019simplifying, li2018deeper, oono2019graph, cai2020note}.
We illustrate this issue in terms of filter approximation power based on the LSS framework.

Suppose a GCN stacks $L$ GC layers $f_{1}^{(l)}(G, \Mat{x}) \in \Set{F}_1$
, apply mapping $\tau$ to its spanned LSS, the isomorphic polynomial space is:
\begin{equation} \label{equ:vanilla_lss}
	\Set{F}_1^L \simeq \left\{p : p(x) = \Theta \sum_{l=0}^{L} \binom{l}{L} x^l, \Theta \in \real \right\},
\end{equation}
where $\Theta = \theta^{(L)} \cdots \theta^{(1)}$.
According to Equation \ref{equ:vanilla_lss}, one can see \textit{no matter how large $L$ is or how a optimizer tunes the parameters $\theta^{(l)}$, $\dim \Set{F}_1^L = 1$, which implies $\Set{F}_1^L$ degenerates to a negligible subspace inside $\Set{A}$}.
GIN \citep{xu2018powerful} disentangles the the weights for neighborhoods and central nodes. We can write this GC layer as ${f}_{GIN}^{(l)}(G, \Mat{x}) = (\theta_1 \Mat{A}(G) + \theta_0 \Mat{I}) \Mat{x}$. The LSS of GIN is isomorphic to the polynomial space:
\begin{equation}
	\Set{F}_{GIN}^L \simeq \left\{ p : p(x) = \prod_{l=1}^{L} \left(\theta_1^{(l)} x + \theta_0^{(l)} \right), \theta_0^{(l)}, \theta_1^{(l)} \in \real \right\}.
\end{equation}
This polynomial space represents all polynomials that can split over the real domain.
However, \textit{ $\Set{F}_{GIN}^L \subsetneq \Set{A}$ since not all polynomials can be factorized into first-order polynomials.}
The expectation number of real roots of a $K$-degree polynomial with zero-mean random coefficients is $\mathbb{E}(N) \sim (2/\pi)\log K$ \citep{maslova1971mean}. When the ambient dimension goes larger, all-real-root polynomials only occupy a small proportion in the ambient space \citep{li2011probability}, which indicates GIN does not have full expressiveness in terms of filter representation either.

\paragraph{Higher-order vs. second-order.} \label{sec:higher_vs_sogc}
Higher-order GCs refer to those polynomial filters with degree larger than three (i.e., $\Set{F}_K, K \ge 3$).
They can model multi-hop GCNs such as \citet{luan2019break, liao2019lanczosnet, abu2019mixhop}.
Compared to SoGCs, higher-order GCs have equivalent expressive power, since they can be reduced to SoGCs.
However, we point out four limitations of adopting higher-order kernels:
1) From our polynomial factorization perspective, fitting graph filters using higher-order GC requires coefficient sparsity, which brings about learning difficulty. \citet{abu2019mixhop} overcomes this problem by adding lasso regularization and extra training procedures. Adopting SoGC can avoid these troubles since decomposition into second-order polynomials results in at most one zero coefficient (See Section \ref{sec:gcn_power}).
2) Eigenvalues of graph adjacency matrices diminish when powered.  This leads to a decreasing numerical rank of $\Mat{A}(G)^k$ and makes aggregating larger-scale information ineffective.
SoGCs can alleviate this problem by preventing higher-order powering operations.
3) Higher-order GC lacks nonlinearity.
SoGCN can bring a better balance between the expressive power of low-level layers and nonlinearity among them.   
4) Multi-hop aggregation consumes higher computational resources (See Table \ref{tbl:comparison}). In contrast, SoGC matches the time complexity of vanilla GCN by fixing the kernel size to two.

\begin{figure}[t!]
	\centering
	\includegraphics[width=0.45\textwidth]{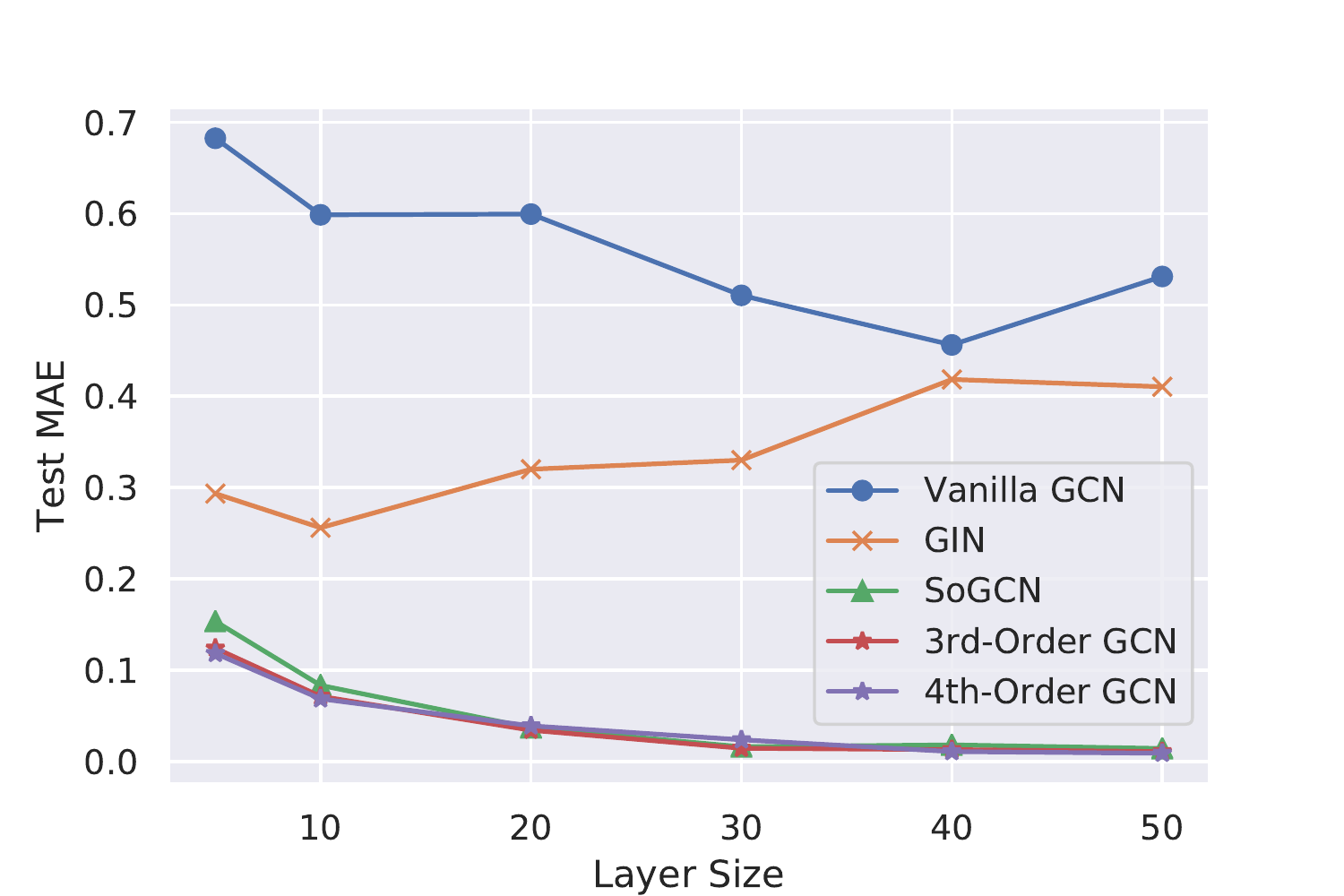}
	\caption{Relations between test MAE and layer size. Experiments are conducted on synthetic \textit{Band-Pass} dataset. Each model has 16 channels per hidden layer with varying layer size.}
	\label{fig:deeper_on_bandpass}
\end{figure}

\subsection{Implementation of Second-Order Graph Convolutional Networks} \label{sec:sogcn}

In this subsection, we introduce other building blocks to establish our Second-Order Graph Convolutional Networks (SoGCN) following the latest trends of GCN.
We promote SoGC to its multi-channel version analogous to \citet{kipf2016semi}.
Then we prepend a feature embedding layer, cascade multiple SoGC layers, and append a readout module.
Suppose the network input is $\Mat{X} \in \real^{N \times D}$ supported in graph $\G{G} \in \Set{G}$, denote the output of $l$-th layer as $\Mat{X}^{(l)} \in \real^{N \times E}$, the final node-level output as $\Mat{Y} \in \real^{N \times F}$, or graph-level output as $\Mat{Y} \in \real^{E}$, we formulate our novel deep GCN built with SoGC (cf. Equation \ref{equ:sogc}) as follows:
\begin{align}
	\label{equ:sogcn_embed} & \Mat{X}^{(0)} = \phi \left( \Mat{X} ; \Mat{\Phi} \right), \\
    \label{equ:sogcn_conv} & \Mat{X}^{(l)} = \sigma \left( f_2^{(l)} \left( \Mat{X}^{(l-1)}; \Mat{\Theta}^{(l)} \right) \right), \\
	\label{equ:sogcn_readout} & \Mat{Y} = \psi \left( \Mat{X}^{(L)} ; \Mat{\Psi} \right),
\end{align}
where $\Mat{\Theta}^{(l)} = \{ \Mat{\Theta}_{i}^{(l)} \in \real^{E \times E} \}_{i=0,1,2}$ are trainable weights for linear filters $f_2^{(l)} \in \Set{F}_2$;
$\phi: \real^{N \times D} \rightarrow \real^{N \times E}$ is an equivariant embedder \citep{maron2018invariant} with parameters $\Mat{\Phi}$; $\sigma: \real^{N \times E} \rightarrow \real^{N \times E}$ is an activation function.
For node-level readout, $\psi: \real^{N \times E} \rightarrow \real^{N \times F}$ can be a decoder (with parameters $\Mat{\Psi}$) or an output activation (e.g., softmax) in place of the prior layer.
For graph-level output, $\psi: \real^{N \times E} \rightarrow \real^{E}$ should be an invariant readout function \citep{maron2018invariant}, e.g., channel-wise sum, mean or max.
In practice, we adopt ReLU as nonlinear activation (i.e., $\sigma = \operatorname{ReLU}$), a multi-layer perceptron (MLP) as the embedding function $\phi$, another MLP for node regression readout, and sum \citep{xu2018powerful} for graph classification readout.

We also provide a variant of SoGCN integrated with Gated Recurrent Unit (GRU) \citep{girault2015translation}, termed SoGCN-GRU.
According to \citet{cho2014learning}, GRU can utilize gate mechanism to preserve and forget information. We hypothesize that a GRU can be trained to remove redundant signals and retain lost features on the spectrum. 
Similar to \citet{li2016gated,gilmer2017neural}, we append a shared GRU module after each GC layer, which takes the signal before the GC layer as the hidden state, after the GC layer as the current input.
We show by our experiment that GRU can facilitate SoGCN in avoiding noises and enhancing features on the spectrum (Figure \ref{fig:vis_spectrum}).
Our empirical study in Table \ref{tab:res_ablation_study} also indicates the effectiveness of GRU for spectral GCNs is general. Hence, we suggest including this recurrent module as another basic building block of our SoGCNs.

\begin{table}[t!]
    \begin{center}
        \caption{The performance of graph node signal regression with \textit{High-Pass}, \textit{Low-Pass}, and \textit{Band-Pass} filters (over graph spectral space) as learning target. Each model has 16 GC layers and 16 channels of hidden layers.}
        \vspace{.08in}
        \label{tab:res_sgs}
        \footnotesize
        \renewcommand{\arraystretch}{1.2}
        \scalebox{0.95}{
        \begin{tabular}{l|c|c|c|c}
            \hline
            \multirow{2}{*}{\textbf{Model}} & \multirow{2}{*}{\textbf{\#Param}} & \multicolumn{3}{c}{Test MAE} \\ \cline{3-5}
            & & \textbf{High-Pass} & \textbf{Low-Pass} & \textbf{Band-Pass} \\
            \hline
            Vanilla & 4611 & 0.308 & 0.317 & 0.559 \\
            GIN & 4627 & 0.344 & 0.096 & 0.274 \\
            SoGCN & 12323 & 0.021 & 0.023 & 0.050 \\
            3rd-Order & 16179 & 0.021 & 0.022 & 0.045 \\
            4th-Order & 20035 & 0.021 & 0.022 & 0.049 \\
            \hline
        \end{tabular}
        }
    \end{center}
\end{table}

\makesavenoteenv{tabular}
\makesavenoteenv{table}
\begin{table*}[t!]
    \begin{center}
        \caption{Results and comparison with other GNN models on ZINC, CIFAR10, MNIST, CLUSTER and PATTERN datasets. For ZINC dataset, the parameter budget is set to 500k. For CIFAR10, MNIST, CLUSTER and PATTERN datasets, the parameter budget is set to 100k. \CBest{Red}: the best model, \CBetter{Green}: good models.}
        \label{tab:res_benchmarks}
        \renewcommand{\arraystretch}{1.2}
        \begin{tabular}{l|}
            \hline
            \multirow{2}{*}{\textbf{Model}} \\
            \\
            \hline
            Vanilla GCN \\
            Vanilla GCN + GRU \\
            GAT \\
            MoNet \\
            GraphSage \\
            GIN \\
            GatedGCN \\
            3WLGNN \\
            \hline
            SoGCN \\
            SoGCN-GRU \\
            \hline
        \end{tabular}
        \begin{tabular}{|c|}
            \hline
            Test MAE $\pm$ s.d. \\ \cline{1-1}
            \textbf{ZINC} \\
            \hline
            0.367$\pm$0.011 \\
            0.295$\pm$0.005 \\
            0.384$\pm$0.007 \\
            0.292$\pm$0.006 \\
            0.398$\pm$0.002 \\
            0.387$\pm$0.015 \\
            0.350$\pm$0.020 \\
            0.407$\pm$0.028 \tablefootnote{~This is the result of 3WLGNN  with 100k parameters. The test MAE of 3WLGNN  with 500k parameters is increased to 0.427$\pm$0.011.} \\
            \hline
            \CBetter{0.238$\pm$0.017} \\
            \CBest{0.201$\pm$0.006} \\
            \hline
        \end{tabular}
        \begin{tabular}{|c|c|c|c}
            \hline
            \multicolumn{4}{|c}{Test ACC $\pm$ s.d. (\%)} \\ \cline{1-4}
            \textbf{MNIST} & \textbf{CIFAR10} & \textbf{CLUSTER} & \textbf{PATTERN} \\
            \hline
            90.705$\pm$0.218 & 55.710$\pm$0.381 & 53.445$\pm$2.029 & 63.880$\pm$0.074 \\
            96.020$\pm$0.090 & 61.332$\pm$0.849 & 57.932$\pm$0.168 & 70.194$\pm$0.216 \\
            95.535$\pm$0.205 & 64.223$\pm$0.455 & 57.732$\pm$0.323 & 75.824$\pm$1.823 \\
            90.805$\pm$0.032 & 65.911$\pm$2.515 & 58.064$\pm$0.131 & 85.482$\pm$0.037 \\
            \CBetter{97.312$\pm$0.097} & 65.767$\pm$0.308 & 50.454$\pm$0.145 & 50.516$\pm$0.001 \\
            96.485$\pm$0.252 & 55.255$\pm$1.527 & 58.384$\pm$0.236 & \CBetter{85.590$\pm$0.011} \\
            \CBetter{97.340$\pm$0.143} & \CBetter{67.312$\pm$0.311} & 60.404$\pm$0.419 & 84.480$\pm$0.122 \\
            95.075$\pm$0.961 & 59.175$\pm$1.593 & 57.130$\pm$6.539 & \CBetter{85.661$\pm$0.353} \\
            \hline
            \CBetter{96.785$\pm$0.113} & \CBetter{66.338$\pm$0.155} & \CBest{68.167$\pm$1.164} & \CBest{85.735$\pm$0.037} \\
            \CBest{97.729$\pm$0.159} & \CBest{68.208$\pm$0.271} & \CBetter{67.994$\pm$2.619} & \CBetter{85.711$\pm$0.047} \\
            \hline
        \end{tabular}
    \end{center}
\end{table*}

\section{Experiments} \label{sec:experiments}

\subsection{Synthetic Graph Spectrum Dataset for Filter Fitting Power Testing} \label{sec:synthetic_exp}

\jshi{To validate the expressiveness of SoGCN, and its power to fit arbitrary graph filters, we build a Synthetic Graph Spectrum (SGS) dataset for the node signal filtering regression task. We construct SGS dataset with random graphs. The learning task is to simulate three types of hand-crafted filtering functions: high-pass, low-pass, and band-pass on the graph spectrum (defined over the graph eigenvectors). There are 1k training graphs, 1k validation graphs, and 2k testing graphs for each filtering function.  Each graph is undirected and comprises 80 to 120 nodes. Appendix \ref{supp:sgs_dataset} covers more details of our SGS dataset.  We choose Mean Absolute Error (MAE) as evaluation metric.
}

\paragraph{Experimental Setup.}

We compare SoGCN with vanilla GCN \citep{kipf2016semi}, GIN \citep{xu2018powerful}, and higher-order GCNs on the synthetic dataset. 
To evaluate each model's expressiveness purely on the GC kernel design, we remove ReLU activations for all tested models. 
We adopt the Adam optimizer \citep{kingma2015adam} in our training process, with a batch size of 128. The learning rate begins with 0.01 and decays by half once the validation loss stagnates for more than 10 training epochs.

\paragraph{Results and Discussion.}

\jshi{
Table~\ref{tab:res_sgs} summarizes the quantitative comparisons.
SoGCN achieves the superior performance on all of the 3 tasks outperforming vanilla GCN and GIN, which implies that SoGC graph convolutional kernel does benefit from explicit disentangling of the  second-hop neighborhoods.  Our results also show that higher-order (3rd-order and 4th-order) GCNs do not improve the performance further, even though they incorporate much more parameters.  SoGCN is more expressive and does a better trade-off between performance and model size.
}

	

Figure~\ref{fig:deeper_on_bandpass} plots MAE results as we vary the depth of GC layers for each graph kernel type. Vanilla GCN and GIN can not benefit from depth while SoGC and higher-order GCs can leverage depth to span larger LSS, contributing to the remarkable filtering results. 
SoGC and higher-order GCs have very close performance after increasing the layer size, which suggests higher-order GCs do not obtain more expressiveness than SoGC. 

\begin{table}[h!]
    \begin{center}
        \caption{The performance of node-level multi-label classification on ogb-protein dataset. We compare each model condering the following dimensions: the number of parameters, training time (in seconds) per epoch (ep.), and final test ROC-AUC (\%).}
        \vspace{.08in}
        \label{tab:res_ogbprotein}
        \footnotesize
        \renewcommand{\arraystretch}{1.2}
        \scalebox{0.95}{
        \begin{tabular}{l|c|c|c}
            \hline
            \multirow{2}{*}{\textbf{Model}} & \multicolumn{3}{c}{\textbf{ogb-protein}} \\ \cline{2-4}
            & \#Param & Time / Ep. & ROC-AUC $\pm$ s.d. \\
            \hline
            Vanilla GCN & 96880 & 3.47 $\pm$ 0.40 & 72.16 $\pm$ 0.55 \\
            GIN & 128512 & 4.33 $\pm$ 0.27 & 76.77 $\pm$ 0.20 \\
            GCNII & 227696 & 4.96 $\pm$ 0.29 & 74.79 $\pm$ 1.17 \\
            GCNII* & 424304 & 5.09 $\pm$ 0.17 & 72.50 $\pm$ 2.49 \\
            APPNP & 96880 & 6.56 $\pm$ 0.37 & 65.37 $\pm$ 1.15 \\
            GraphSage & 193136 & 6.51 $\pm$ 0.13 & 77.53 $\pm$ 0.30 \\
            \hline
            SoGCN & 192512 & 4.88 $\pm$ 0.36 & \textbf{79.28 $\pm$ 0.47} \\
            4th-Order GCN & 320512 & 8.89 $\pm$ 0.82 & 78.95 $\pm$ 0.57 \\
            6th-Order GCN & 448512 & 9.76 $\pm$ 0.64 & 78.61 $\pm$ 0.42 \\
            \hline
        \end{tabular}
        }
    \end{center}
\end{table}

\begin{table*}[t!]
\begin{center}
    \caption{Results of ablation study on ZINC, MNIST and CIFAR10 datasets. Vanilla GCN is the comparison baseline and the number in the ($\uparrow~\boldsymbol{\cdot}$) and ($\downarrow~\boldsymbol{\cdot}$) represents the performance gain compared with the baseline.}
    \label{tab:res_ablation_study}
    \begin{tabular}{l|l|l|l}
        \hline
        \multirow{2}{*}{\textbf{Model}} & \multicolumn{1}{c|}{Test MAE $\pm$ s.d.} & \multicolumn{2}{c}{Test ACC $\pm$ s.d. (\%)} \\
        \cline{2-4} & \multicolumn{1}{c|}{\textbf{ZINC}} & \multicolumn{1}{c|}{\textbf{MNIST}} & \multicolumn{1}{c}{\textbf{CIFAR10}} \\
        \hline
        Vanilla GCN & $0.367 \pm 0.011$~(Baseline) & $90.705 \pm 0.218$~(Baseline)  & $55.710 \pm 0.381$~(Baseline) \\
        SoGCN & $0.238\pm0.017$~($\downarrow$~\textbf{0.129}) & $96.785\pm0.113$~($\uparrow$~\textbf{6.080}) & $66.338\pm0.155$~($\uparrow$~\textbf{10.628}) \\
        4th-Order GCN & $0.243\pm0.009$~($\downarrow$~0.124) &$96.167\pm0.198$~($\uparrow$~5.462) & $64.230\pm0.212$~($\uparrow$~8.520)\\
        6th-Order GCN & $0.261\pm 0.014$~($\downarrow$~0.106) & $96.292\pm 0.134$~($\uparrow$~5.587) & $63.687\pm 0.151$~($\uparrow$~7.977) \\
        \hline
        Vanilla GCN + GRU & $0.295 \pm 0.005~(\downarrow0.072)$ & $96.020 \pm 0.090~(\uparrow5.315)$ & $61.332 \pm 0.381~(\uparrow5.622)$\\
        SoGCN + GRU & $0.201\pm0.006$~($\downarrow$~\textbf{0.166}) & $97.729\pm0.159$~($\uparrow$~\textbf{7.024}) & $68.208\pm0.271$~($\uparrow$~\textbf{12.498}) \\
        4th-Order GCN + GRU & $0.204\pm0.004~(\downarrow0.163)$ & $97.304\pm0.296~(\uparrow6.599)$ & $64.697\pm0.341~(\uparrow8.987)$ \\
        6th-Order GCN + GRU & $0.218\pm 0.005~(\downarrow 0.149)$ & $97.325\pm 0.218~(\uparrow 6.620)$ & $64.523\pm 0.251~(\uparrow 8.813)$  \\
        \hline
    \end{tabular}
\end{center}
\end{table*}

\subsection{OGB Benchmarks} \label{sec:ogb_exp}

We choose Open Graph Benchmark (OGB) \citep{hu2020ogb} to compare our SoGC with other GCNs in terms of the parameter numbers, train time per epoch, and test ROC-AUC. We only demonstrate the results for predicting presence of protein functions (multi-label graph classification). We refer interested reader to Appendix \ref{supp:more_expr} for more results on OGB.

\paragraph{Experiment Setup} The chosen models mainly include spectral-domain models: vanilla GCN, GIN, APPNP \citep{klicpera2018predict}, GCNII \citep{chenWHDL2020gcnii}, our SoGCN, and two high-order GCNs. We also obtain the performance of GraphSage, the vertex-domain GNN baseline, for a reference. We build GIN, GCNII, GraphSage, and APPNP based on official implementations in PyTorch Geometric \citep{Fey/Lenssen/2019}.
Every model consists of three GC layers, and the same node embedder and readout modules.
We borrow the method of \citet{dwivedi2020benchmarking} to compute the number of paramters.
We run an exclusive training program on an Nvidia Quadro P6000 GPU to test the training time per epoch.
We follow the same training and evaluation procedures on the OGB benchmarks to ensure fair comparisons. We train each model until convergence (\textasciitilde1k epochs for vanilla GCN, GIN, GraphSage, and \textasciitilde3k epochs for SoGCN, higher-order GCNs, APPNP, GCNII).

\paragraph{Results and Discussion.} Table \ref{tab:res_ogbprotein} demonstrates the ROC-AUC score for each model on ogb-protein dataset. Our SoGCN achieves the best performance among all presented GCNs but its parameter number and time complexity is only slightly higher than GIN (consistent with Table \ref{tbl:comparison}).
SoGC is more expressive than other existing graph filters (such as APPNP and GCNII), and also outperforms message-passing GNN baseline GraphSage.
Compared with higher-order (4th and 6th) GCNs, the ROC-AUC score of SoGCN surpasses all of them while reducing model complexity significantly. 

\subsection{GNN Benchmarks}
\label{sec:benchmarks_exp}



We follow the benchmarks outlined in ~\citet{dwivedi2020benchmarking} for evaluating GNNs on several datasets across a variety of artificial and real-world tasks.
We choose to evaluate our SoGCN on a real-world chemistry dataset (ZINC molecules) for the graph regression task, two semi-artificial computer vision datasets (CIFAR10 and MNIST superpixels) for the graph classification task, and two artificial social network datasets (CLUSTER and PATTERN) for node classification.


\paragraph{Experimental Setup.} 


We compare our proposed SoGCN and SoGCN-GRU with state-of-the-art GNNs: vanilla GCN, GIN, GraphSage, GAT \citep{velivckovic2017graph}, MoNet \citep{monti2017geometric}, GatedGCN \citep{bresson2017residual} and 3WL-GNN \citep{maron2019provably}. To ensure fair comparisons, we follow the same training and evaluation pipelines (including optimizer settings) and data splits of benchmarks.   Furthermore, we adjust our model's depth and width to ensure it satisfies parameter budgets as specified in the benchmark. 
Note that we do not use any geometrical information to encode rich graph edge relationship, as in models such as GatedGCN-E-PE.  We only employ  graph connectivity information for all tested models.

\paragraph{Results and Discussion.}

Table~\ref{tab:res_benchmarks} reports the benchmark results. Our model SoGCN makes small computational changes to GCN by adopting second-hop neighborhood, and it outperforms models with complicated message-passing mechanisms, such as GAT and GraphSage.  With GRU module, SoGCN-GRU tops almost all state-of-the-art GNNs on the ZINC, MNIST and CIFAR10 datasets.
In Figure \ref{fig:vis_spectrum}, we visualize a spectrum of the last layer's feature activation on ZINC dataset. One can see our SoGC can extract features on high-frequency bands and GRU can further sharpen these patterns. 
However, GRU does not lift accuracy on CLUSTER and PATTERN datasets for node classification task. According to \citet{li2018deeper}, that GRU suppresses low-frequency band results in the slight performance drop on the CLUSTER and PATTERN datasets.

\paragraph{Ablation Study.}

To contrast the performance gain produced by different aggregation ranges and GRU on the benchmarks, we evaluate vanilla GCN, SoGCN, 4th-Order GCN, 6th-Order GCN as well as their GRU variants on the ZINC, MNIST and CIFAR10 datasets. Table~\ref{tab:res_ablation_study} presents the results of our ablation study, which are consistent with our observation on Section \ref{sec:synthetic_exp} and \ref{sec:ogb_exp}. As shown by our ablation study, adopting the second-hop aggregation makes huge performance gain (vanilla GCN vs. SoGCN). However, high-order GCNs are not capable of boosting the performance further over SoGCN. On the contrary, higher-order GCs can even lead to the performance drop (4th-Order GCN vs. 6th-Order GCN vs. SoGCN). We also testify GRU's effectiveness for each presented model. But the gain brought by GRU is not as large as adding second-hop aggregation. Figure \ref{fig:vis_spectrum} shows our SoGC can extract patterns on the spectrum alone. GRU plays a role of enhancing the features.

\section{Conclusion} \label{sec:conclusion}

What should be the basic convolutional blocks for GCNs? To answer this, we seek the most localized graph convolution kernel (GC) with full expressiveness.  We establish our LSS framework to assess GC layers of different aggregation ranges.   We show the second-order graph convolutional filter, termed SoGC, possesses the full representation power than one-hop GCs. Hence, it becomes the efficient and simplest GC building blocks that we adopt to establish our SoGCN. Both synthetic and benchmark experiments exhibit the prominence of our theoretic design.
We also make an empirical study on the GRU's effects in spectral GCNs.
Interesting directions for future work include analyzing two-hop aggregation schemes with message-passing GNNs and proving the universality of nonlinear GCNs.







\bibliography{references}
\bibliographystyle{icml2021}

\appendix

\section{Remark on Definition 1} \label{supp:remark_def1}

Let us rewrite the $\Set{F}_K$ following Definition 1:
\begin{equation*}
	\Set{F}_K = \left\{ f: f(G, \Mat{x}) = \sum_{k = 0}^{K} \theta_k \Mat{A}(\G{G})^k \Mat{x}, \forall \theta_k \in \real \right\}.
\end{equation*}
We claim that functions $f \in \Set{F}_K: \Set{G} \times \real^N \rightarrow \real^N$ are all Linear Shift-Invariant (LSI) to adjacency matrix.
\begin{proof}
Given arbitrary graph $\G{G} \in \Set{G}$, any filter $\Mat{H}$ associated with it can be written as below:
\begin{equation*}
	\Mat{H}(G) = \sum_{k = 0}^{K} \theta_k \Mat{A}(G)^k = \Mat{U} \left( \sum_{k = 0}^{K} \theta_k \Mat{\Lambda}^k \right) \Mat{U}^T,
\end{equation*}
where $\Mat{A}(G) = \Mat{U} \Mat{\Lambda} \Mat{U}^T$ is the eigendecomposition of $\Mat{A}(G)$.
Therefore, $\Mat{H}(G)$ is also diagonalized by the eigenvectors of $\Mat{A}(G)$. By the Lemma \ref{lemma:simul_diag}:
\begin{lemma} \label{lemma:simul_diag}
	Diagonalizable matrices $\Mat{A}_1$ and $\Mat{A}_2$ are simultaneously diagonalized if and only if $\Mat{A}_1 \Mat{A}_2 = \Mat{A}_2 \Mat{A}_1$.
\end{lemma}
we say that $\Mat{H}(G)$ commutes with $\Mat{A}(G)$. For any $f \in \Set{F}_K$, $\Mat{A}(G)f(G, \Mat{x}) = f(G, \Mat{A}(G) \Mat{x})$.
\end{proof}

\section{Ring Isomorphism $\pi: \Set{F}_K \rightarrow \Set{T}_K$} \label{supp:dev_ring_iso}

We introduce a mathematical device that bridges the gap between the filter space $\Set{F}_K$ and the polynomial space $\real_{K}[x]$.

Since $\Set{G}$ is finite, we can construct a block diagonal matrix $\Mat{T} \in \real^{N\lvert\Set{G}\rvert \times N\lvert\Set{G}\rvert}$, with adjacency matrix of every graph on the diagonal:
\begin{equation} \label{equ:block_diag}
	\Mat{T} = \begin{bmatrix}
		\Mat{A}(\G{G}_1) & & \\
		& \ddots & \\
		& & \Mat{A}(\G{G}_{\lvert\Set{G}\rvert}) \\
	\end{bmatrix} \in \real^{N\lvert\Set{G}\rvert \times N\lvert\Set{G}\rvert}.
\end{equation}

\begin{remark}
The spectrum capacity $\Capacity$ in Definition 2 represents the number of eigenvalues of $\Mat{T}$ without multiplicity.
\end{remark}

Eigenvalues of adjacency matrices signify graph similarity.
The spectrum capacity $\Capacity$ identifies a set of graphs by enumerating the structural patterns.
Even if the graph set goes extremely large (to guarantee the generalization capability), the distribution of spectrum provides the upper bound of $\Capacity$, so our theories remain their generality.

Now we construct a matrix space $\Set{T}_K$ by applying a ring homomorphism $\pi: \Set{F}_K \rightarrow \Set{T}_K$ to every element in $\Set{F}_K$:
\begin{equation}
	\pi: \sum_{k = 0}^{K} \theta_k \Mat{A}(\G{G})^k \mapsto \sum_{k = 0}^{K} \theta_k \Mat{T}^k.
\end{equation}
Concretely, we write the matrix space $\Set{T}$ as follows:
\begin{equation} \label{equ:matrix_poly}
	\Set{T}_K = \left\{ \Mat{H}: \Mat{H} = \sum_{k = 0}^{K} \theta_k \Mat{T}^k, \forall \theta_k \in \real \right\}.
\end{equation}

In the rest section, we prove that $\pi$ is a ring isomorphism.

\begin{figure*}[t!] 
	\centering
	\includegraphics[width=\textwidth]{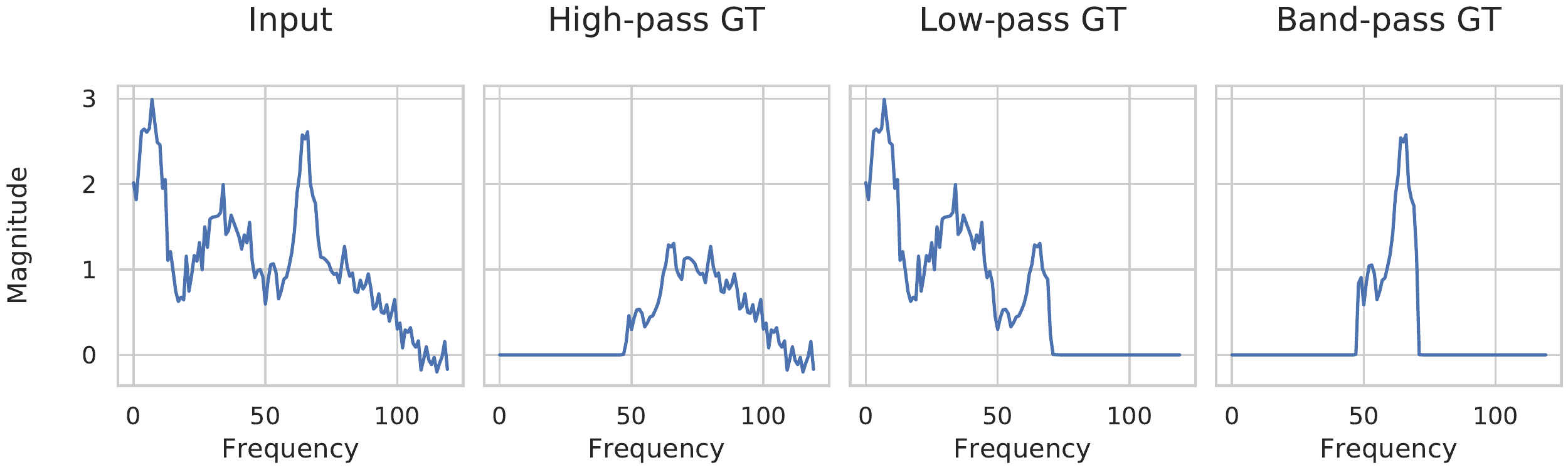}
	\caption{An example of graph spectrum in our SGS dataset and its corresponding high-pass, low-pass and band-pass filtered output using our hand-crafted filters.}
	\label{fig:vis_sgs_dataset}
\end{figure*}

\begin{proof}
First, we can verify that $\pi$ is a ring homomorphism because it is invariant to ``summation" and ``multiplication".
Second, we can prove its surjectivity by the definition of $\Set{T}_K$ (cf. Equation \ref{equ:matrix_poly}).

Finally, we show its injectivity as follows:
Consider any pair of $f_1, f_2 \in \Set{F}_K, f_1 \ne f_2$ with parameters $\alpha_k, \beta_k \in \real, k=0, \cdots, K$, there exists $\G{G}_j \in \Set{G}$ and $\Mat{x} \in \real^{N}$ such that $f_1(\G{G}_j, \Mat{x}) \ne f_2(\G{G}_j, \Mat{x})$.
After applying $\pi$, we have their images $\Mat{H}_1 = \pi(f_1), \Mat{H}_2 = \pi(f_2)$.
Let $\Mat{\xi} = \begin{bmatrix} \Mat{0}_{N(j-1)}^T & \Mat{x}^T & \Mat{0}_{N(\lvert \Set{G} \rvert - j)}^T \end{bmatrix}^T$,
where $\Mat{0}_N$ denote the all-zero vector of length $N$, then we have:
\begin{align*}
	\Mat{H}_1 \Mat{\xi} &= \begin{bmatrix}
		\Mat{0}_{N(j-1)}^T & \left( \sum_{k = 0}^{K} \alpha_k \Mat{A}(\G{G}_j)^k \Mat{x} \right) ^T & \Mat{0}_{N(\lvert \Set{G} \rvert - j)}^T
	\end{bmatrix}^T  \\
	&= \begin{bmatrix}
		\Mat{0}_{N(j-1)}^T & f_1(\G{G}_j, \Mat{x})^T & \Mat{0}_{N(\lvert \Set{G} \rvert - j)}^T
	\end{bmatrix}^T,
\end{align*}
\begin{align*}
	\Mat{H}_2 \Mat{\xi} &= \begin{bmatrix}
		\Mat{0}_{N(j-1)}^T & \left( \sum_{k = 0}^{K} \beta_k \Mat{A}(\G{G}_j)^k \Mat{x} \right) ^T & \Mat{0}_{N(\lvert \Set{G} \rvert - j)}^T
	\end{bmatrix}^T \\
	&= \begin{bmatrix}
		\Mat{0}_{N(j-1)}^T & f_2(\G{G}_j, \Mat{x})^T & \Mat{0}_{N(\lvert \Set{G} \rvert - j)}^T
	\end{bmatrix}^T.
\end{align*}
Hence, $\Mat{H}_1 \ne \Mat{H}_2$ concludes the injectivity.
\end{proof}

\section{Proof of Lemma 1} \label{supp:prf_dim}


\begin{proof}
One can show $\Set{F}_K$ is a vector space by verifying the linear combination over $\Set{F}_K$ is closed (or simply implied from the ring isomorphism $\pi$).

Due to isomorphism, $\dim \Set{F}_K = \dim \Set{T}_K$.
Then Lemma 1 follows from Theorem 3 of \citet{sandryhaila2013discrete}. We briefly conclude the proof as below.

Let $m(x)$ denote the minimal polynomial of $\Mat{T}$. We have $\Capacity = \deg m(x)$.
Suppose $K+1 < \Capacity$. First, $\dim \Set{T}_K$ cannot be larger than $K+1$, because $\left\{ \Mat{I}, \Mat{T}, \cdots, \Mat{T}^K \right\}$ is a spanning set. If $\dim \Set{T}_K < K+1$, then there exists some polynomial $p(x)$ with $\deg p(x) < K$, such that $p(\Mat{T}) = \Mat{0}$. This contradicts the minimality of $m(x)$. Therefore, $\dim \Set{T}_K$ can only be $K+1$.

Suppose $K+1 \ge \Capacity$. For any $\Mat{H} = h(\Mat{T})$ where polynomial $h(x)$ has $\deg h(x) \le K$. By polynomial division, there exists unique polynomials $q(x)$ and $r(x)$ such that
\begin{align} \label{equ:poly_division}
	h(x) = q(x) m(x) + r(x),
\end{align}
where $\deg r(x) < \deg m(x) = \Capacity$. We insert $\Mat{T}$ into Equation \ref{equ:poly_division} as below:
\begin{align*}
	h(\Mat{T}) = q(\Mat{T}) m(\Mat{T}) + r(\Mat{T}) = q(\Mat{T}) \Mat{0} + r(\Mat{T}) = r(\Mat{T}).
\end{align*}
Therefore, $\left\{ \Mat{I}, \Mat{T}, \cdots, \Mat{T}^{\Capacity - 1} \right\}$ form a basis of $\Set{T}_K$, i.e., $\dim \Set{T}_K = \Capacity$.
\end{proof}

\section{Proof of Lemma 2} \label{supp:prf_ring_iso}
\begin{proof}
Consider a mapping $\varphi: \Set{T}_K \rightarrow \real_{K}[x]$:
\begin{equation}
    \varphi: \sum_{k = 0}^{K} \theta_k \Mat{T}^k \mapsto \sum_{k = 0}^{K} \theta_k x^k.
\end{equation}
When $K+1 \le \Capacity$, $\dim \Set{T}_K = \dim \real_K[x]$ (as $\deg m(x) = \Capacity$), which implies $\varphi$ is a ring isomorphism as well.
Since function composition preserves isomorphism property, we can conclude the proof by showing that $\tau = \varphi \circ \pi$.
\end{proof}

\begin{remark}
The assumption that each graph has the same number of vertices is made only for the sake of simplicity. Lemma 1 and Lemma 2 still hold when the vertex numbers are varying, since the construction of $\Mat{T}$ (cf. Equation \ref{equ:block_diag}) is independent of this assumption.
\end{remark}

\begin{remark}
The graph set $\Set{G}$ need to be finite, otherwise $\Capacity$ might be uncountable. We leave the discussion on infinite graph sets for future study.
\end{remark}


\section{Synthetic Graph Spectrum Dataset} \label{supp:sgs_dataset}

Our Synthetic Graph Spectrum (SGS) dataset is designed for testing the filter fitting power of spectral GCNs. It includes 3 types of graph signal filters: High-Pass (HP), Low-Pass (LP) and Band-Pass (BP) filters. For each type, we generate 1k, 1k and 2k undirected graphs along with graph signals and groundtruth response in training set, validation set and test set, respectively. Each graph has 80\textasciitilde 120 nodes and 80\textasciitilde 350 edges. Models are trained on each dataset to learn the corresponding filter by supervision on the MAE loss.



For each sample, we generate an undirected Erd\H{o}s-R\'enyi random graph $G = (\Set{V}, \Set{E})$ with normalized adjacency matrix $\Mat{A}$, i.e., the existence of the edge between each pair of nodes accords to a Bernoulli distribution $\operatorname{Bern}(p)$. In our experiments, we set $p = 0.02$ to satisfy $\lvert\Set{E}(\G{G})\rvert = O(N)$. We also compute $\Mat{L}=\Mat{I}-\Mat{A}=\Mat{U}\Mat{\Lambda}\Mat{U^T}$, where $\Mat{\Lambda}=\diag(\lambda_1, \dots, \lambda_N)$ with $\lambda_N \geq \cdots \geq \lambda_1$ are eigenvalues, $\Mat{U}=\begin{bmatrix} \Mat{u}_1 & \cdots & \Mat{u}_N \end{bmatrix}$ are corresponding eigenvectors.

\begin{figure*}[t!]
	\begin{subfigure}[b]{1.0\textwidth}
		\centering
		\includegraphics[width=1.0\textwidth]{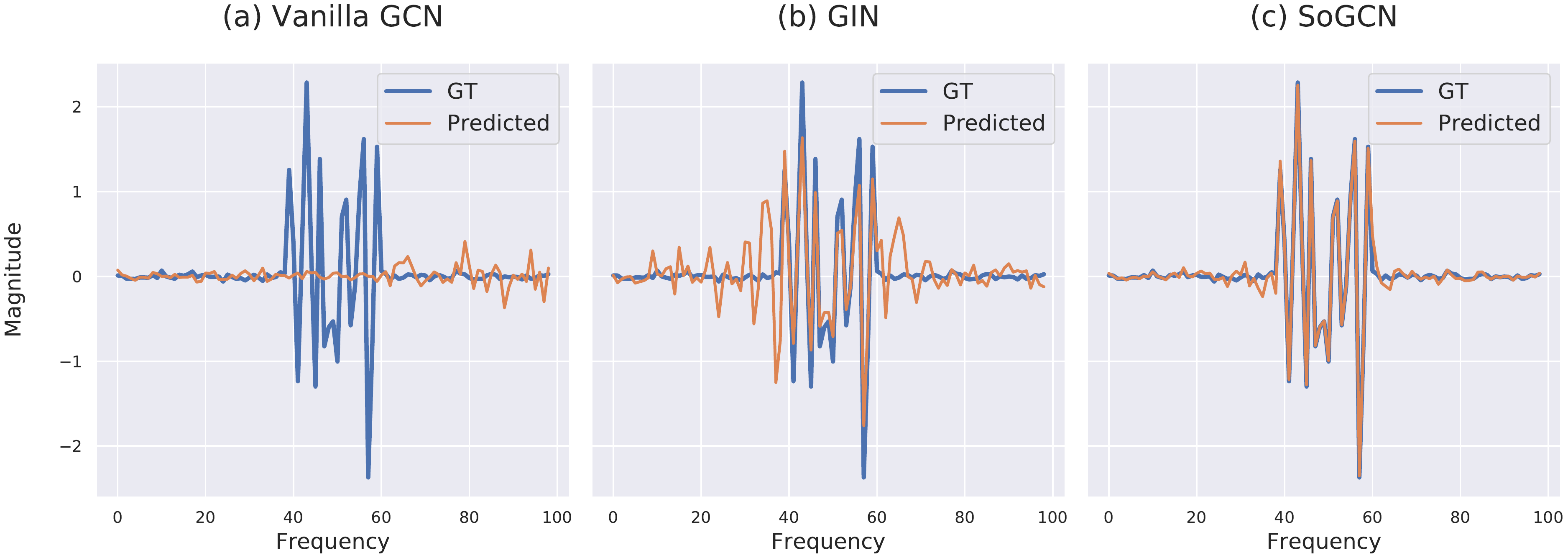}
	\end{subfigure}
	
	\caption{Visualize the spectrum of outputs from vanilla GCN, GIN and SoGCN on the SGS Band-Pass dataset.}
	\label{fig:vis_sgs_output_spectrum}
\end{figure*}

Next, we generate input graph signals $\Mat{s}$ on the spectral domain. Independent sampling for each frequency from a distribution tends to generate  white noises. Hence, we synthesize spectrum by summing random functions. We notice the mixture of beta function $\operatorname{Beta}(a,b)$ and Gaussian function $\operatorname{Norm}(\mu, \sigma)$ is a powerful model to construct diverse curves by tuning shape parameters $(a, b)$ and $(\mu, \sigma)$.
We sum two discretized beta functions and four discretized Gaussian functions with random parameters to generate signal spectrums. Equation \ref{equ:sgs_sig_gen} elaborates the generation process and hyper-parameter chosen in our experiments, where $g[x; a, b]$ is the PDF of $\operatorname{Beta}(a,b)$ distribution, $f[x; \mu, \sigma]$ denotes the PDF of $\operatorname{Norm}(\mu, \sigma)$ distribution.
\begin{equation} \label{equ:sgs_sig_gen}
    \begin{split}
    & \Mat{s}_t = \sum_{i=1}^{2} g[t/N; a_i, b_i] + \sum_{j=1}^{4}c_j f[t; \mu_j, \sigma_j], t\in [N] \\
    & a_i, b_i \sim \operatorname{Unif}\{0.1, 5\}, \quad \mu_j\sim \operatorname{Unif}\{0, N\}, \\
    & \sigma_j\sim \left. \operatorname{Unif}\left\{\frac{N}{(j+1)}, \frac{N}{j}\right\} \middle/9 \right., \\
    & c_j \sim \frac{\operatorname{Unif}\{0.5, 2\}}{\operatorname{max}_{x\in [N]}f[x; \mu_j, \sigma_j]},
    \end{split}
\end{equation}
We can retrieve the vertex-domain signals via inverse graph Fourier transformation: $\Mat{\hat{x}} = \Mat{U}\Mat{s}$. Then Gaussian noise is added to the vertex-domain signals to simulate observation errors: $\Mat{x} = \Mat{\hat{x}} + \epsilon, \epsilon\sim \operatorname{Norm(0, c)}, c\sim\operatorname{Unif}(0.05, 0.35)$.

We design three filters $\Mat{F}_{HP}^*$, $\Mat{F}_{LP}^*$, $\Mat{F}_{BP}^*$ in Equation \ref{equ:sgs_filters}:
\begin{equation} \label{equ:sgs_filters}
    \begin{split}
    & f_{HP}^*(\Mat{s}) = \frac{1}{1 + \zeta(\Mat{s};50, 1)}, \\
    & f_{LP}^*(\Mat{s}) = 1 - \frac{1}{1 + \zeta(\Mat{s};50, 1)} \\
    & f_{BP}^*(\Mat{s}) = \frac{-1}{1 + \zeta(\Mat{s};100, 1.05)}  + \frac{1}{1 + \zeta(\Mat{s};100, 0.95)}, \\\\
    & \Mat{F}_{k}^* = \Mat{U}f_{k}^*(\Mat{\Lambda})\Mat{U^T}, k\in\{HP, LP, BP\},
    \end{split}
\end{equation}
where $\zeta(\Mat{s};\alpha, \beta) = \operatorname{exp}\{-\alpha(\Mat{s}-\beta)\}$. For supervising purpose, we applying each filter to synthetic inputs to generate the groundtruth output: $\Mat{y} = \Mat{F}_{k}^*\Mat{x}, k\in\{HP,LP,BP\}$.
Figure \ref{fig:vis_sgs_dataset} illustrates an example of the generated spectral signals and the groundtruth responses of three filters.

\section{More Visualizations of Spectrum} \label{supp:more_vis_spectrum}

For multi-channel node signals $\Mat{X} \in \mathbb{R}^{N\times D}$, where $N$ is the number of nodes and $D$ is the number of signal channels, the spectrum of $\Mat{X}$ is computed by $\Mat{S} = \Mat{U}^T\Mat{X}$. More information about the graph spectrum and graph Fourier transformation can be found in \citet{sandryhaila2013discrete}.

Figure~\ref{fig:vis_sgs_output_spectrum} shows the output spectrum of vanilla GCN, GIN and SoGCN on the synthetic Band-Pass dataset. The visualizations are consistent with the results in Table~2 and Figure~3 in the main text. Vanilla GCN almost loses all the band-pass frequency, resulting in very poor performance. GIN learns to pass a part of middle-frequency band but still has a distance from the groundtruth. SoGCN's filtering response is close to the groundtruth response, showing its strong ability to represent graph signal filters.

We arbitrarily sample graph data from the ZINC dataset as input and visualize the output spectrum of vanilla GCN, SoGCN and their GRU variants in Figure~\ref{fig:more_vis_spectrum}. Each curve in the visualization figure represents the spectrum of each output channel, i.e., each column of $\Mat{S}$ is plotted as a curve.

\begin{figure*}[t!] 
	\begin{subfigure}[b]{1.0\textwidth}
		\centering
		\includegraphics[width=0.95\textwidth]{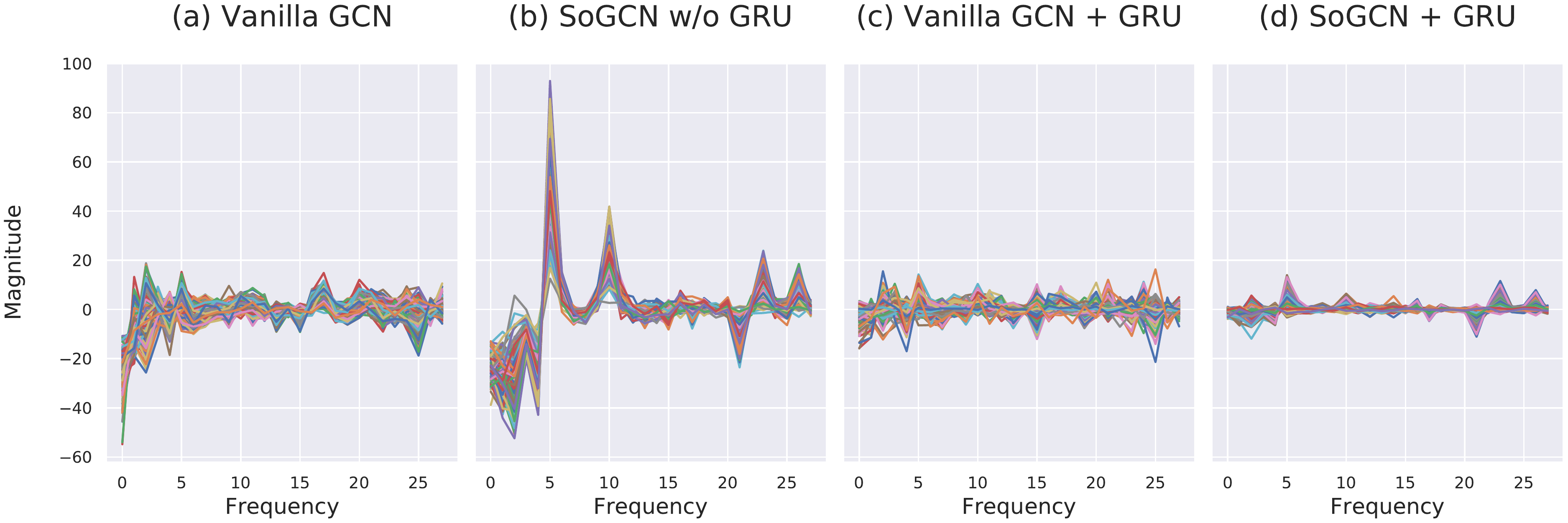}
	\end{subfigure}
	\begin{subfigure}[b]{1.0\textwidth}
		\centering
		\includegraphics[width=0.95\textwidth]{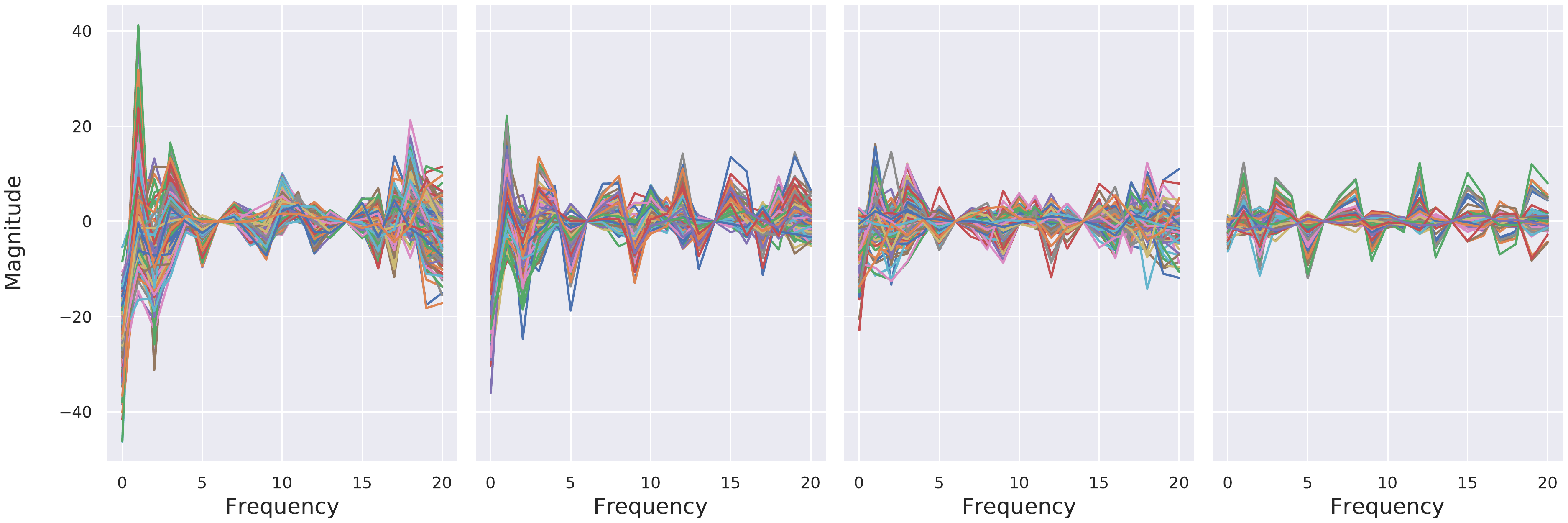}
	\end{subfigure}
	\begin{subfigure}[b]{1.0\textwidth}
		\centering
		\includegraphics[width=0.95\textwidth]{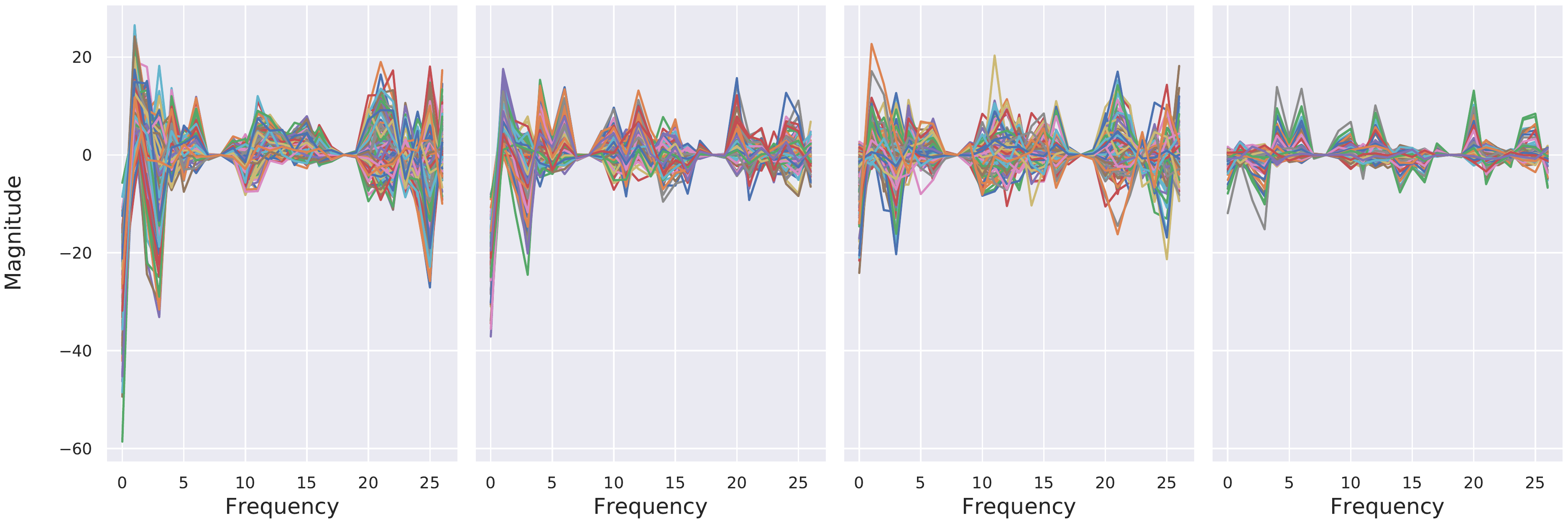}
	\end{subfigure}
	
	\caption{More visualizations of output spectrum on the ZINC dataset.}
	\label{fig:more_vis_spectrum}
\end{figure*}

\section{More Experiments} \label{supp:more_expr}

\begin{figure*}[t!] 
	\begin{subfigure}[b]{0.5\textwidth}
		\centering
		\includegraphics[width=0.9\textwidth]{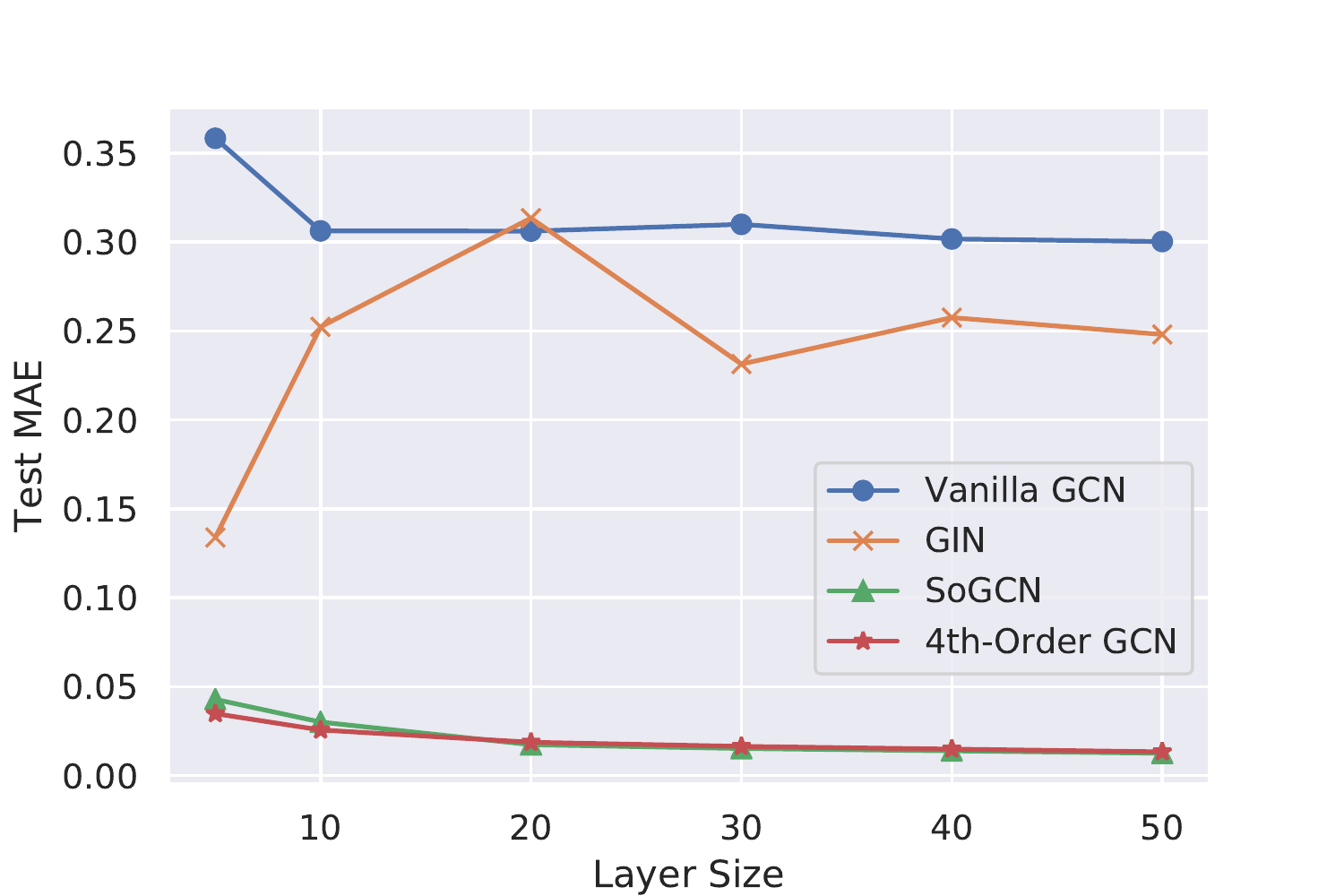}
		\caption{Relation on High-Pass dataset}
	\end{subfigure}
	\begin{subfigure}[b]{0.5\textwidth}
		\centering
		\includegraphics[width=0.9\textwidth]{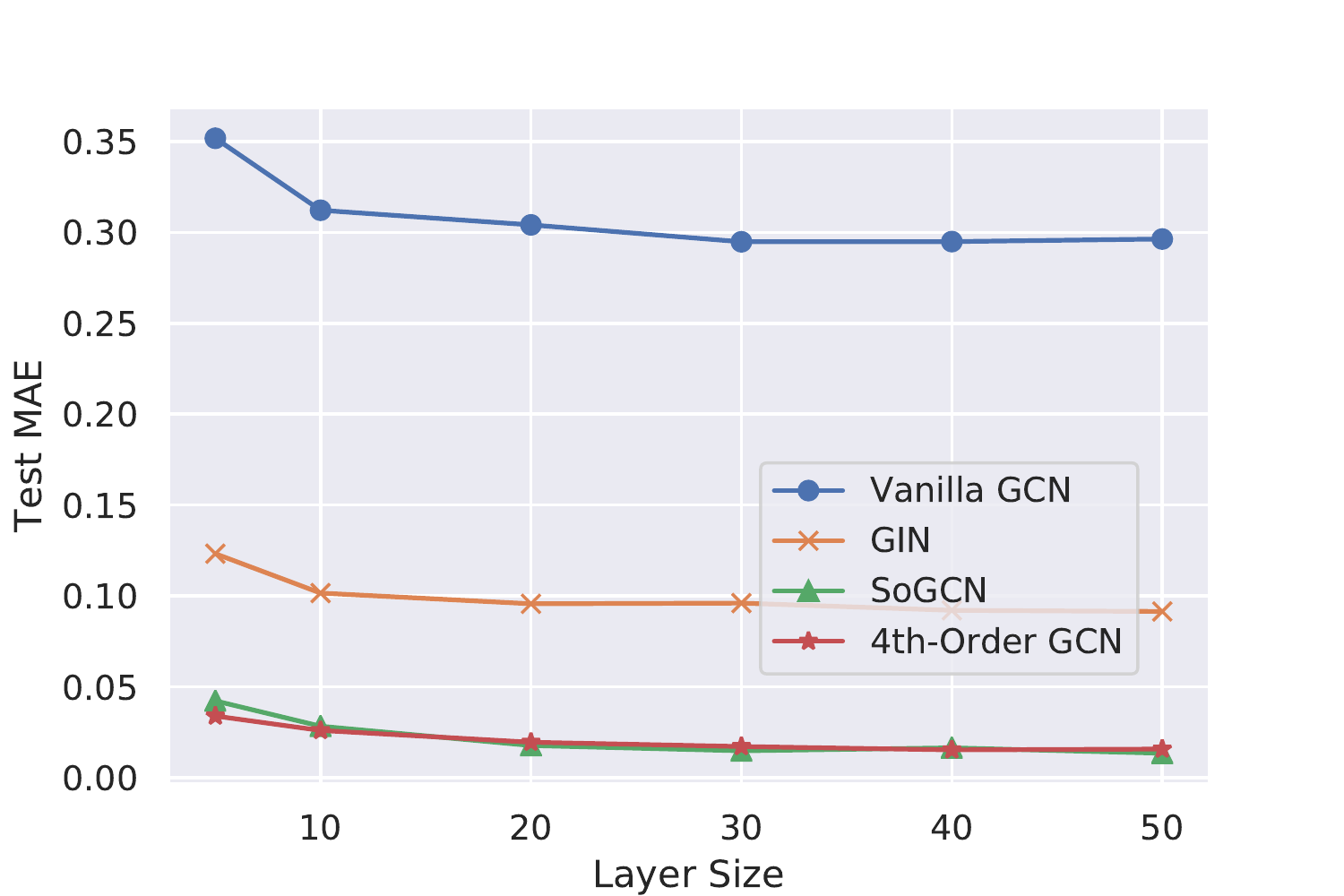}
		\caption{Relation on Low-Pass dataset}
	\end{subfigure}

	\caption{Relations between test MAE and layer size. Each model has 16
channels per hidden layer with varying layer size.}
	\label{fig:deeper_on_lowpass_highpass}
\end{figure*}

\subsection{Additional Experiments on SGS Dataset}

We supplement two experiments to compare vanilla GCN \citep{kipf2016semi}, GIN \citep{xu2018powerful}, SoGCN and 4th-order GCN on synthetic High-Pass and Low-Pass datasets, respectively. With Figure \ref{fig:deeper_on_lowpass_highpass}, we conclude that SoGCN and high-order GCNs perform closely on High-Pass and Low-Pass datasets and achieve remarkable filtering capability, while vanilla GCN and GIN cannot converge to considerable results by increasing the layer size. This conclusion is consistent with the previous results on Band-Pass dataset presented in the main text.

\subsection{Additional Experiments on OGB Benchmark}

In the main text, we have demonstrated our results on ogb-protein dataset \citep{hu2020ogb} for node-level tasks. In this subsection, we also show our SoGCN's effectiveness on ogb-molhiv dataset \citep{hu2020ogb} for graph-level tasks. As the same with experiments on ogb-protein, we evaluate different GCN models in terms of their total parameter numbers, training time per epoch, and test ROC-AUC.

\paragraph{Experiment Setup} Again, we choose vanilla GCN, GIN, GraphSage\citep{hamilton2017inductive}, APPNP \citep{klicpera2018predict}, GCNII \citep{chenWHDL2020gcnii}, ARMA \citep{bianchi2019graph}, our SoGCN, and two high-order GCNs for performance comparison.
We adopt the example code of vanilla GCN and GIN provided in OGB. We reimplemented GCNII, GraphSage, APPNP, ARMA based on the official code in PyTorch Geometric \citep{Fey/Lenssen/2019}.
According to the benchmark's guideline, we add edge features to fan-out node features while propagation.
Every model has the same depth and width, as well as other modules.
The timing, training and evaluation procedures conform with the descriptions in our main text. We train vanilla GCN, GIN, APPNP, GraphSage for \textasciitilde100 epochs, and SoGCN, higher-order GCNs, GCNII, ARMA for \textasciitilde500 epochs.

\begin{table}[t!]
    \begin{center}
        \caption{The performance of graph-level multi-label classification on ogb-molhiv dataset. We compare each model condering the following dimensions: the number of parameters, training time (in seconds) per epoch (ep.), and final test ROC-AUC (\%).}
        \vspace{.08in}
        \label{tab:res_ogbmolhiv}
        \footnotesize
        \renewcommand{\arraystretch}{1.2}
        \scalebox{0.95}{
        \begin{tabular}{l|c|c|c}
            \hline
            \multirow{2}{*}{\textbf{Model}} & \multicolumn{3}{c}{\textbf{ogb-molhiv}} \\ \cline{2-4}
            & \#Param & Time / Ep. & ROC-AUC $\pm$ s.d. \\
            \hline
            Vanilla GCN & 527,701 & 25.57 $\pm$ 1.37 & 76.06 $\pm$ 0.97 \\
            GIN & 980,706 & 29.01 $\pm$ 1.24 & 75.58 $\pm$ 1.40 \\
            GCNII & 524,701 & 24.19 $\pm$ 1.26 & 77.04 $\pm$ 1.03 \\
            APPNP & 327,001 & 13.56 $\pm$ 1.32 & 68.00 $\pm$ 1.36 \\
            GraphSage & 976,201 & 24.43 $\pm$ 1.39 & 76.90 $\pm$ 1.36 \\
            ARMA & 8,188,201 & 43.14 $\pm$ 0.99 & 76.91 $\pm$ 1.75 \\
            \hline
            SoGCN & 1,426,201 & 27.02 $\pm$ 1.28 & \textbf{77.26 $\pm$ 0.85} \\
            4th-Order GCN & 2,326,201 & 32.24 $\pm$ 1.10 & 77.24 $\pm$ 1.21 \\
            6th-Order GCN & 3,226,201 & 37.64 $\pm$ 1.15 & 77.10 $\pm$ 0.72 \\
            \hline
        \end{tabular}
        }
    \end{center}
\end{table}

\paragraph{Results and Discussion.} Table \ref{tab:res_ogbmolhiv} demonstrates the ROC-AUC score for each model on ogb-molhiv dataset. We reach the same conclusion with our main text.
On ogb-molhiv dataset, we notice that GCNII is another lightweight yet effective model. However, GCNII only allows inputs whose channel number equals to output dimension. One needs to add additional blocks (e.g., linear modules) to support varying hidden dimensions, which incorporates more parameters and higher complexity (e.g., on ogb-protein dataset).

\section{Implementation Details} \label{supp:impl_details}

We open source our implementation of SoGCN at \href{https://github.com/yuehaowang/SoGCN}{https://github.com/yuehaowang/SoGCN}. All of our code, datasets, hyper-parameters, and runtime configurations can be found there.

\subsection{Second-Order Graph Convolution}

Our SoGC can be implemented using a message-passing scheme \citep{hamilton2017inductive} (cf. Equation \ref{equ:sogc_mp}). We regard the normalized adjacency matrix $\Mat{A}(\G{G})$ as a one-hop aggregator (message propagator). When we compute the power of $\Mat{A}(\G{G})$, we invoke the propagator multiple times. After passing the messages twice, we transform and mix up aggregated information from two hops via a linear block.
\begin{align} \label{equ:sogc_mp}
    \begin{split}
    & \Mat{h}_v^{(1)} = \sum_{u \in \Set{N}(v)} \frac{1}{\sqrt{d_v d_u}} \Mat{x}_u, \\
    & \Mat{h}_v^{(2)} = \sum_{u \in \Set{N}(v)} \frac{1}{\sqrt{d_v d_u}} \Mat{h}_u^{(1)}, \\
    & \Mat{y}_v = \Mat{\Theta}_2 \Mat{h}_v^{(2)} + \Mat{\Theta}_1 \Mat{h}_v^{(1)} + \Mat{\Theta}_0 \Mat{x}_v,
    \end{split}
\end{align}
where $\Mat{x}_v \in \real^{E}$ is the input feature vector for node $v \in \Set{V}(\G{G})$, $\Mat{y}_v \in \real^{F}$ denotes the output for node $v$. $d_v$ is the degree for vertex $v$, $\Set{N}(v)$ is the set of $v$'s neighbor vertices. $\Mat{h}_v^{(1)}$ is the feature representation of $v$'s first-hop neighborhood. It can be computed by aggregating information once from the directly neighboring nodes. $\Mat{h}_v^{(2)}$ is the feature representation of $v$'s second-hop neighborhood. It can be computed by feature aggregation upon neighbors' $\Mat{h}_u^{(1)}, u \in \Set{N}(v)$. $\Mat{\Theta}_i \in \real^{F \times E}, i = 0, 1, 2$ are the weight matrices (a.k.a. layer parameters).

Our design can reduce computational time by reusing previously aggregated information and preventing power operations on $\Mat{A}(\G{G})$. 
In practice, our SoGC is easy to implement. Our message-passing design conforms to mainstream graph learning frameworks, such as Deep Graph Library \citep{wang2019dgl} and PyTorch Geometric \citep{Fey/Lenssen/2019}.
One can simply add another group of parameters and invoke the ``propagation" method of vanilla GC \citep{kipf2016semi} twice to simulate our SoGC. For the sake of clarity, we provide the pseudo-code for general $K$-order GCs in Algorithm \ref{alg:sogc}. Our SoGC can be called by passing $K = 2$.

\begin{algorithm}[h!]
\caption{$K$-Order Graph Convolution}
\label{alg:sogc}
\begin{algorithmic}
\STATE {\bfseries Input:} Graph $\G{G} = (\Set{V}, \Set{E})$, node degrees $\{ d_v \in \mathbb{N}, \forall v \in \Set{V} \}$, input features $\{\Mat{x}_v \in \real^E, \forall v \in \Set{V} \}$, GC order $K$, weight matrices $\Mat{\Theta}_i \in \real^{F \times E}, i = 0, \cdots K$.
\STATE {\bfseries Output:} Feature representation $\Mat{y}_v \in \real^F, \forall v \in \Set{V}$.
\STATE $\Mat{h}^{(0)}_v \gets \Mat{x}_v, \forall v \in \Set{V}$
\STATE $\Mat{y}_v \gets \Mat{\Theta}_0 \Mat{h}^{(0)}_v, \forall v \in \Set{V}$
\FOR{$t=1$ {\bfseries to} $K$}
    \FOR{$v \in \Set{V}$}
        \STATE $\left. \Mat{h}_v^{(t)} \gets \sum_{u \in \Set{N}(v)} \Mat{h}_u^{(t-1)} \middle/ {\sqrt{d_v d_u}} \right. $
    \ENDFOR
    \STATE $\Mat{y}_v \gets \Mat{y}_v + \Mat{\Theta}_t \Mat{h}^{(t)}_v, \forall v \in \Set{V}$
\ENDFOR
\STATE {\bfseries Return} $\{ \Mat{y}_v \in \real^F, \forall v \in \Set{V} \}$
\end{algorithmic}
\end{algorithm}

\subsection{Gated Recurrent Unit}
\label{supp:gru}

We supplement two motivations behind using Gated Recurrent Unit (GRU) \citep{cho2014learning}:
1) GRU has been served as a basic building block in message-passing GNN architectures \citep{li2016gated, gilmer2017neural}. We make an explorative attempt to first introduce them into spectral GCNs.
2) By selectively maintaining information from previous layer and canceling the dominance of DC components (Figure \ref{fig:more_vis_spectrum}), GRU can also relieve the side-effect of ReLU, which is proved to be a special low-pass filter \citep{oono2019graph, cai2020note}. 

Similar to \citet{li2016gated, gilmer2017neural}, we appends a shared GRU module after each GC layer, which takes the signals before the GC layers as the hidden state, after the GC layers as the current input. We formulate its implementation by replacing Equation 10 with Equation \ref{equ:sogcn_gru} as below.
\begin{equation} \label{equ:sogcn_gru}
    \begin{split}
    &\Mat{X}_{conv}^{(l)} = f_2^{(l)} \left( \Mat{X}^{(l-1)}; \Mat{\Theta}^{(l)} \right) \\
	& \Mat{X}^{(l+1)} = \operatorname{GRU}\left( \operatorname{ReLU}\left(\Mat{X}_{conv}^{(l+1)}\right), \Mat{X}^{(l)}; \Mat{\Omega} \right),
    \end{split}
\end{equation}
where $\Mat{X}_{conv}^{(l+1)}$ is the input, $\Mat{X}^{(l)}$ represents the hidden state, $\Mat{\Omega}$ denotes parameters of the GRU.
Figure \ref{fig:more_vis_spectrum} illustrates the spectrum outputs of vanilla GCN + GRU and SoGCN + GRU. One can see, without filtering power of SoGCN, vanilla GCN + GRU fails to extract sharp patterns on the spectrum. Thereby, we suggest that it is SoGC that mainly contributes to the higher expressiveness.

\end{document}